\documentclass[a4paper]{article}

\usepackage{geometry, quiver, tikz-cd}
\usepackage{svg}
\usepackage[utf8]{inputenc}
\usepackage[T1]{fontenc}

\usepackage{amsmath, amsfonts, bm, amssymb, commath, amsthm}
\usepackage[bookmarksnumbered=true]{hyperref}
\usepackage[dvipsnames]{xcolor}
\usepackage[font=small,labelfont=bf]{caption}
\usepackage{enumitem}
\usepackage[backend=bibtex,natbib,style=authoryear,maxcitenames=2,maxbibnames=99]{biblatex}
\bibliography{bib.bib}

\theoremstyle{plain}
\newtheorem{theorem}{Theorem}
\newtheorem{proposition}[theorem]{Proposition}
\newtheorem{lemma}[theorem]{Lemma}

\newtheorem*{definition*}{Definition}
\newtheorem{corollary}[theorem]{Corollary}

\theoremstyle{remark}
\newtheorem{remark}{Remark}

\renewcommand{\cal}[1]{\mathcal{#1}}
\newcommand{\bb}[1]{\mathbb{#1}}
\renewcommand{\r}{\mathbb{R}}

\newcommand{\KL}{\textup{KL}}
\newcommand{\pg}{\textup{PG}}
\newcommand{\ball}[2]{B_{#2}(#1)}

\DeclareMathOperator*{\argmin}{arg\,min}

\usepackage{xspace}

\newcommand{\PL}{P{\L}\xspace}

\DeclareMathOperator*{\osc}{osc}
\DeclareMathOperator*{\diag}{diag}
\DeclareMathOperator*{\logit}{logit}
\DeclareMathOperator{\lip}{Lip}
\DeclareMathOperator*{\Ber}{Ber}
\DeclareMathOperator*{\tr}{Tr}
\DeclareMathOperator{\sech}{sech}
\DeclareMathOperator*{\Ric}{Ric}

\hypersetup{
  colorlinks   = true, 
  urlcolor     = {black!50!gray}, 
  linkcolor    = {black!50!gray}, 
  citecolor   = {black!50!gray} 
}
\title{ \vspace{-1.75cm} Wasserstein Contraction of \\ Coordinate Ascent Variational Inference}

\author{Rocco Caprio\footnotemark[2] \and Adrien Corenflos\footnotemark[3] \and Sam Power\footnotemark[4]}

\begin{document}
\maketitle
\begin{abstract}
    We study the non-asymptotic contraction in Wasserstein distance of the sequential, parallel, and random-scan coordinate ascent variational inference algorithms. This is shown to hold under a functional smoothness condition of the optimality maps and a transportation-information inequality at their fixed points. Our results are sharp and general, and as opposed to those based on global strong log-concavity assumptions, they allow for local convergence on smooth, non-smooth, and discrete manifolds, including within the context of data augmentation.
    We consider many applications in statistical physics and Bayesian statistics. These include pairwise Markov Random field models such as Ising and Curie--Weiss, unbalanced Bayesian Gaussian Mixture Models, high-dimensional Bayesian Probit Regression, and high-dimensional Logistic Regression with P\'olya--Gamma random variables (i.e.~Jaakkola--Jordan's algorithm). In many of these models, these represent the first available convergence results of their kind. \\
    
\end{abstract}

\begingroup
\renewcommand{\thefootnote}{\fnsymbol{footnote}}
\footnotetext[2]{University of Warwick and Bocconi University. \href{mailto:rocco.caprio@unibocconi.it}{rocco.caprio@unibocconi.it}}
\footnotetext[3]{University of Warwick. \href{mailto:adrien.corenflos@warwick.ac.uk}{adrien.corenflos@warwick.ac.uk}}
\footnotetext[4]{University of Bristol. \href{mailto:sam.power@bristol.ac.uk}{sam.power@bristol.ac.uk}}
\endgroup

\section{Introduction}

A key problem in scientific computing is approximating an intractable probability distribution $\pi$ of interest, usually known only up to a normalising constant.
Variational Inference stands out as a particularly attractive tool for this task, owing to its statistical and computational efficiency, and it has been the framework underlying many advances in computational statistics over the past half century~\citep{Parisi1980,Hinton1993,Jordan1999,Bishop2006}.
The central idea is to seek a tractable approximation to $\pi$ within a chosen family of tractable distributions $\cal{Q}$ by optimising a divergence to $\pi$ over that \textit{variational family}. Often, we define optimality through minimisation of the Kullback--Leibler (KL) divergence (also `relative entropy'):
\begin{align*}
  \mu_\star=\argmin \left\{\KL(\mu\parallel\pi): \mu\in \cal{Q} \right\}.
\end{align*}
A key practical aspect of working with this particular loss function is that in solving the associated optimisation problem, one is only required to compute expectations under the tractable variational distribution $\mu$, rather than under the intractable target $\pi$. Often, it is convenient or well-motivated to work with the family of product (or tensor, or factorised) distributions $\cal{Q}=\cal{P}^{\otimes M}$. In these cases, we seek the best \textit{mean-field approximation} $\mu_\star$ to $\pi$ that solves
\begin{align*} 
  \mu_\star=\mu_\star^1\otimes \dots \otimes \mu_\star^M = \argmin \left\{ \KL(\mu^1\otimes \dots \otimes \mu^M\parallel\pi): \mu^i\in\cal{P}(\cal{X}^i), i=1,\dots,M \right\}.
\end{align*}
The coordinate ascent variational inference algorithm~\citep[CAVI,][]{Bishop2006,Blei2017} solves this problem by per-coordinate minimisation of the relative entropy. There are various ways to do so: 
the sequential, or systematic scan, CAVI algorithm iteratively minimises the Kullback--Leibler divergence in a systematic fashion: given the current iterate $(\mu_k^i)$, $i=1,\dots,M$, it computes
\begin{align*} 
	\mu_{k+1}^i = \argmin_{\mu\in\cal{P}(\cal{X}^i)} \KL(\mu_{k+1}^1 \otimes \dots \otimes \mu_{k+1}^{i-1} \otimes \mu \otimes \mu_k^{i+1} \otimes \dots \otimes \mu_k^M \parallel \pi), \quad i=1,\dots,M.
\end{align*}
The parallel CAVI algorithm instead computes the new values in block given the previous iterates
\begin{align*} 
	\mu_{k+1}^i = \argmin_{\mu\in\cal{P}(\cal{X}^i)} \KL(\mu_{k}^1 \otimes \dots \otimes \mu_{k}^{i-1} \otimes \mu \otimes \mu_k^{i+1} \otimes \dots \otimes \mu_k^M \parallel \pi), \quad i=1,\dots,M.
\end{align*}
and as such, is more amenable to parallelisation. Finally, the random CAVI algorithm selects a random coordinate $I\in \{1,\dots,M\}$, and updates it as
\begin{align*}
	\mu_{k+1}^I = \argmin_{\mu\in\cal{P}(\cal{X}^I)} \KL(\mu_{k}^1 \otimes \dots \otimes \mu_{k}^{I-1} \otimes \mu \otimes \mu_k^{I+1} \otimes \dots \otimes \mu_k^M \parallel \pi), \quad \mu_{k+1}^{i}=\mu_k^i, \quad \forall i\neq I.
\end{align*}

\paragraph{Contributions.}
In this work, we study the contraction, in the sense of the Wasserstein distance, of the coordinate ascent variational inference iterates $(\mu_k^1\otimes\dots\otimes \mu_k^M)_{k\geq 0}$ of the sequential, parallel, and random algorithms, towards the solution of the variational inference problem, $\mu_\star=\mu_\star^1\otimes\dots\otimes \mu_\star^M$.
The contraction is established under two conditions: a functional \emph{smoothness} assumption, and a \emph{transportation-information inequality} at the limiting fixed-point marginals $\mu_\star^i$, $i=1,\dots,M$. The results are general, sharp, and the assumptions are verifiable. The analysis allows for multiple limiting points depending on the initialisation of the algorithm. 
In the setting of \emph{data augmentation}, where one of the spaces, say $\cal{X}^M$, houses auxiliary variables introduced solely for computational purposes, and statistical interest lies only in $(\mu_k^1\otimes\dots\otimes \mu_k^{M-1})_{k\geq 0}$, we show that convergence of these alone can be established under a weaker functional smoothness condition, and a transportation-information inequality on $(\mu_\star^i)$, $i=1,\dots,M-1$ only. This avoids making any direct requirement on the nature of $\cal{X}^M$ or the fixed point $(\mu_\star^M)$.
This is particularly valuable when $\cal{X}^M$ is non-smooth or even discrete (as in e.g.\ Gaussian Mixture Models) or when verifying transportation-information conditions on the fixed point of $(\mu_\star^M)$ is difficult.

We apply our convergence results to study the coordinate ascent variational inference in some statistical physics and Bayesian models, for which we obtain the first such results we are aware of. In particular, we study the CAVI algorithms on Unbalanced Gaussian Mixture Models, where we show that the convergence of CAVI is governed by a posterior phase transition, which explicitly depends on the posterior modes separation; Bayesian Probit and Bayesian Logistic Regression~\citep[the Jaakkola--Jordan algorithm]{Jaakkola2000}, where we also study the asymptotic properties of the obtained rates and bounds in the high-dimensional limit; and on discrete and continuous spin systems, including the Ising and Curie--Weiss models, the Gaussian free field, and perturbations thereof, including the Ginzburg--Landau $\Phi_4$ model and some related models of spatial statistics. We close with a brief discussion of dynamical models (i.e.\ Markov and Hidden Markov) and some opportunities for future applications.

\paragraph{Related works and comparisons.}

While coordinate ascent variational inference is a fundamental algorithm in statistics and machine learning, non-asymptotic convergence guarantees have only recently started to emerge.
The first such guarantee for the two-block setting appears in \citet{Bhattacharya2025}, who prove contraction in the symmetrised KL divergence under a novel notion of a generalised correlation coefficient.
Shortly after, an alternative line of research adopted an optimisation perspective: under the assumption of strong log-concavity and block-smoothness, \citet{Arnese2024, Lavenant2024} analyse systematic and random-scan CAVI, establishing convergence of the CAVI iterates in relative entropy. The approach based on log-concavity and block-smoothness of the target $\pi$ has since been extended to structured variational inference~\citep{Sheng2025} and to `polyhedral' variational families~\citep{Jiang2025}.
Some model or purpose-specific analyses of CAVI have also appeared, including \citet{Plummer2020} for bivariate Ising models, \citet{Datta2025} for Bayesian principal component analysis, and \citet{Pati2026} who study the convergence of CAVI for linear regression, with a focus on the comparison between sequential and parallel implementations.

Our results are instead phrased in terms of functional inequalities and functional smoothness conditions, and we prove direct contraction in the Wasserstein distance of the coordinate ascent variational inference iterates for a general number of blocks. We provide a sharp, direct, and short proof of convergence for the sequential, the parallel, and the random algorithms. The extensive literature on functional inequalities provides many flexible tools for practical verification that our assumptions hold in specific examples.
The relaxation from log-concavity of $\pi$ to functional inequality-based assumptions \emph{on the fixed points only} is crucial for practical applicability of our results: standard models (including the logistic regression and mixture models that we consider in this paper) yield target distributions that are not strongly log-concave.  Moreover, our results accommodate the possibility of multiple limiting points, local convergence, and analysis in discrete spaces (including the Ising and the Curie--Weiss models), features which are excluded by prevalent strong log-concavity assumptions. Moreover, our bounds and rates are sharp, as opposed to those obtained in most of these previous works.

Finally, in terms of closely related algorithms, we mention \citet{Yao2024, Jiang2025}, who analyse the global convergence of some mean-field variational algorithms under a quadratic growth assumption on the relative entropy and strong log-concavity, respectively; as well as \citet{Caprio2025, Ascolani2025}, who analyse the global convergence of the EM algorithm and the Gibbs sampler, under some novel functional inequalities and strong log-concavity assumptions.

\paragraph{Outline.} In Section~\ref{sec:pialtmin}, we study the Wasserstein contraction of the coordinate ascent variational inference iterates $(\mu_k^1\otimes\dots\otimes \mu_k^M)_{k\geq 0}$ towards the solution of the mean-field variational inference problem $\mu_\star=\mu_\star^1\otimes\dots\otimes\mu_\star^M$, for the sequential, the parallel, and the random schemes. First, in Section~\ref{sec:2blocks}, we illustrate the ideas and the approach for the sequential scheme in the two-block $M=2$ case, where it is particularly transparent. In Section~\ref{sec:Mblocks}, we investigate the three coordinate ascent variational inference  algorithms for a general number $M$ of blocks, we compare their bounds, and discuss differences. In Section~\ref{sec:hps} we discuss in detail our assumptions, and we compare them with those in the literature. In Section~\ref{sec:discrete} we show how the same ideas, with a sufficient level of abstraction, extend straightforwardly to models set on discrete spaces, demonstrating that this approach of analysis by methods of optimal transport is not confined to the smooth setting. 

In Section~\ref{sec:appls}, we apply our convergence results to different models in statistical physics and Bayesian inference: we study Gaussian Mixture Models in Section~\ref{sec:gmm}, Bayesian classification with probit and logit links in Section~\ref{sec:classi}, and pairwise Markov random fields models in~\ref{sec:pmrf}, for both discrete and continuous spin systems, including the Ising, the Curie--Weiss, and the Ginzburg--Landau settings, as well as Hidden Markov Models and Latent Gaussian Spatial Models.

\section{Coordinate Ascent Variational Inference} \label{sec:pialtmin}

Let $\pi=e^{-f}\in\cal{P}(\times_{i=1}^M\cal{X}^i)$ be a target probability distribution for which we wish to form a tractable approximation. Coordinate ascent variational inference attempts to solve the mean-field variational inference problem
\begin{align} \label{eq:minimisation}
	 \mu_\star = \argmin \left\{ \KL(\mu^1\otimes \dots \otimes \mu^M\parallel\pi): \mu^i\in\cal{P}(\cal{X}^i), i=1,\dots,M \right\}.
\end{align}
$\KL$ denotes the relative entropy functional, or Kullback--Leibler divergence, defined as
\begin{align*}
	\KL(\varrho_1\parallel\varrho_2) = \int \log\left(\frac{\dif \varrho_1}{\dif \varrho_2}(x) \right) \varrho_1(\dif x) \text{ if } {\varrho_1\ll\varrho_2} \text{ and } \KL(\varrho_1\parallel\varrho_2)=\infty \text{, otherwise}.
\end{align*}
As detailed previously, there are multiple ways to perform coordinate descent on this objective function: the sequential, or systematic scan, CAVI algorithm iteratively minimises the Kullback--Leibler divergence with respect to one factor at a time, as
\begin{align} \label{eq:seqcavi}
	\mu_{k+1}^i = \argmin_{\mu\in\cal{P}(\cal{X}^i)} \KL(\mu_{k+1}^1 \otimes \dots \otimes \mu_{k+1}^{i-1} \otimes \mu \otimes \mu_k^{i+1} \otimes \dots \otimes \mu_k^M \parallel \pi), \quad i=1,\dots,M.
\end{align}
The parallel CAVI algorithm, instead, only uses the previous iterates to compute the new values simultaneously across blocks
\begin{align} \label{eq:parcavi}
	\mu_{k+1}^i = \argmin_{\mu\in\cal{P}(\cal{X}^i)} \KL(\mu_{k}^1 \otimes \dots \otimes \mu_{k}^{i-1} \otimes \mu \otimes \mu_k^{i+1} \otimes \dots \otimes \mu_k^M \parallel \pi), \quad i=1,\dots,M.
\end{align}
and, as such, is more amenable to parallelisation. Finally, the random CAVI algorithm selects a coordinate $I\in \{1,\dots,M\}$ uniformly at random, and updates only this coordinate, i.e.
\begin{align} \label{eq:rancavi}
	\mu_{k+1}^I = \argmin_{\mu\in\cal{P}(\cal{X}^I)} \KL(\mu_{k}^1 \otimes \dots \otimes \mu_{k}^{I-1} \otimes \mu \otimes \mu_k^{I+1} \otimes \dots \otimes \mu_k^M \parallel \pi), \quad \mu_{k+1}^{i}=\mu_k^i, \quad \forall i\neq I.
\end{align}
\paragraph{Preliminaries.} We require the following analytical tools and definitions.
Let $\cal{X}$ be a smooth, complete, Riemannian manifold (for instance, $\r^d$, or a closed smooth subset of it), with geodesic distance $\mathsf{d}$. We let $\cal{P}_p(\cal{X})$ be the space of probability measures on $\cal{X}$ that are absolutely continuous with respect to the standard volume measure $\dif x$ on $\cal{X}$, and that possess finite $p$ moments. For two probability measures $\varrho_1,\varrho_2\in\cal{P}_p(\cal{X})$, we define the relative Fisher Information functional as
\begin{align} \label{eq:fisher}
  \cal{I}(\varrho_1\parallel\varrho_2) = \int \Big| \nabla \log \frac{\dif \varrho_1}{\dif \varrho_2}(x) \Big|^2 \varrho_1(\dif x) \text{ if } {\varrho_1\ll\varrho_2} \text{ and } \cal{I}(\varrho_1\parallel\varrho_2)=\infty \text{, otherwise}.
\end{align}
In the above, $|\nabla f|$ denotes the norm of the Riemannian gradient.  This work leverages the Wasserstein distance and the related geometry of optimal transport~\citep{Villani2009}. For two Riemannian manifolds $\cal{X}$ and $\cal{X}'$ and $(\varrho_1,\varrho_2)\in\cal{P}_p(\cal{X})\times\cal{P}_p(\cal{X}')$, we denote $\Gamma(\varrho_1,\varrho_2)$ the set of couplings of $\varrho_1$ and $\varrho_2$, i.e.\ of joint distributions on $\cal{X}\times\cal{X}$ admitting $\varrho_1$ and $\varrho_2$ as marginals. We define the Wasserstein-$p$ distance on the space of probability distributions $\cal{P}_p(\cal{X})$:
\begin{align} \label{eq:wassy}
  \mathbb{W}_p(\varrho_1,\varrho_2):= \inf \left\{\int \mathsf{d}(x_1,x_2)^p \gamma(\dif x_1,\dif x_2):\gamma\in\Gamma(\varrho_1,\varrho_2)\right\}^{1/p}.
\end{align}
Under this geometry, we write $B_{\varepsilon}(\varrho) := \{\mu\in\cal{P}_p(\cal{X}):\bb{W}_p(\mu,\varrho)<\varepsilon\}$ for the open ball of radius $\varepsilon$ around $\varrho$, making the dependency on the order $p$ implicit as it will be clear from context.

\subsection{Wasserstein contraction of the two-block algorithm} \label{sec:2blocks}
In this section, we prove Wasserstein contraction of the two-block sequential coordinate ascent variational inference algorithm:
\begin{align} \label{eq:cavi2blocks}
	\mu^1_{k+1} = \argmin_{\mu^1\in\cal{P}(\cal{X}^1)} \KL( \mu^1\otimes\mu_k^2 \parallel \pi), \quad \mu^2_{k+1} = \argmin_{\mu^2\in\cal{P}(\cal{X}^2)} \KL( \mu^1_{k+1}\otimes\mu^2 \parallel\pi).
\end{align}
for which the analysis is particularly transparent.
Define operators $\mu_\star^1[\cdot]:\cal{P}(\cal{X}^2)\mapsto \cal{P}(\cal{X}^1)$ and $\mu_\star^2[\cdot]:\cal{P}(\cal{X}^1)\mapsto \cal{P}(\cal{X}^2)$ by
\begin{equation} \label{eq:optimality2blocks}
  \begin{aligned}
    \mu_\star^1[\nu^2](x_1)&\propto \exp\left(-\mathbb{E}_{\nu^2} f(x_1,X_2) \right), \quad
    \mu_\star^2[\nu^1](x_2)\propto \exp\left(-\mathbb{E}_{\nu^1}f(X_1,x_2)\right).
  \end{aligned}
\end{equation}
These operators are connected to the optimality conditions of~\eqref{eq:cavi2blocks}: since 
\begin{align*}
	\KL(\mu^1\otimes \mu^2\parallel\pi) = \bb{E}_{\mu^1\otimes \mu^2}[f(X_1,X_2)] + \int \log(\mu^1(x_1)\mu^2(x_2))\mu^1(\dif x_1)\mu^2(\dif x_2)
\end{align*}
from the first order conditions of~\eqref{eq:cavi2blocks} we have $\mu_k^1=\mu_\star^1[\mu_{k-1}^2]$ and $\mu_k^2=\mu_\star^2[\mu_k^1]$.

We study the contraction, in the sense of the Wasserstein distance, towards fixed points satisfying
\begin{align}\label{eq:fixedpointeqn}
  \mu_\star^1 = \mu_\star^1[\mu_\star^2] =\mu_\star^1[\mu_\star^2[\mu_\star^1[\dots]]], \qquad \mu^2_\star=\mu_\star^2[\mu_\star^1]=\mu_\star^2[\mu_\star^1[\mu_\star^2[\dots]]].
\end{align}
\begin{figure}[!tbp]
	\centering
	\begin{tikzcd}[cells={nodes={draw}}]
	{\mu_0^1} & {\mu_1^1} & \cdots & \cdots & {\mu_\star^1} \\
\\
{\mu_0^2} & {\mu_1^2} & \cdots & \cdots & {\mu_\star^2}
\arrow["{{{\mu_\star^2[\cdot]}}}"{description}, from=1-1, to=3-2]
\arrow["{{{\mu_\star^2[\cdot]}}}"{description}, from=1-2, to=3-3]
\arrow[from=1-3, to=3-4]
\arrow[from=1-4, to=3-5]
\arrow["{{{\mu_\star^2[\cdot]}}}"{description}, shift right, curve={height=-18pt}, from=1-5, to=3-5]
\arrow["{{{\mu_\star^1[\cdot]}}}"{description}, from=3-1, to=1-1]
\arrow["{{{\mu_\star^1[\cdot]}}}"{description}, from=3-2, to=1-2]
\arrow[from=3-3, to=1-3]
\arrow[from=3-4, to=1-4]
\arrow["{{{\mu_\star^1[\cdot]}}}"{description}, from=3-5, to=1-5]
	\end{tikzcd}
	\caption{Illustration of the sequential coordinate ascent variational inference algorithm with two blocks: we prove convergence provided the optimality maps, represented by the arrows, are smooth, and the fixed points (the rightmost boxes) satisfy a transportation-information inequality.}
	\label{fig:cavi2blocks}
\end{figure}
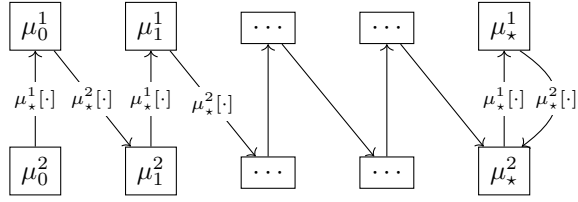
In general, these fixed points are not unique, and convergence towards either depends on the initialisation $\mu_0^2$; we leave this dependence implicit. Our main result says that coordinate ascent variational inference converges rapidly if the optimality maps $\mu_\star^1[\cdot]:\cal{P}(\cal{X}^2)\mapsto \cal{P}(\cal{X}^1)$ and $\mu_\star^2[\cdot]:\cal{P}(\cal{X}^1)\mapsto \cal{P}(\cal{X}^2)$ are smooth around the fixed points $\mu_\star^1=\mu_\star^1[\mu_\star^2[\mu_\star^1[\dots]]]\in\cal{P}(\cal{X}^1)$, and $\mu_\star^2=\mu_\star^2[\mu_\star^1[\mu_\star^2[\dots]]]\in\cal{P}(\cal{X}^2)$, and these fixed-point distributions satisfy a transportation-information inequality. Let $p\geq 1$ be given, we define (local) smoothness for an operator between probability distributions as follows.

\begin{definition*}[Fisher-smoothness]\label{def:crosssmooth2blocks}
  Let $\cal{X}$ and $\cal{X}'$ be smooth Riemannian manifolds. An operator $\gamma_\star[\cdot]:\cal{P}(\cal{X})\mapsto \cal{P}(\cal{X}')$ is said to be Fisher-smooth with constant $L$ at $\ball{\nu}{\varepsilon}$ if, for all $\mu\in \ball{\nu}{\varepsilon}$, there holds the Lipschitz-type estimate
  \begin{align}\label{eq:Icrosssmooth2blocks}
    \cal{I}(\gamma_\star[\mu]\parallel\gamma_\star[\nu]) \leq L^2 \bb{W}_p^2(\mu, \nu).
  \end{align}
\end{definition*}
Notice the analogy with the local (log) smoothness of a function $h:\r^d\mapsto \r$, which requires $|\nabla \log h(x)-\nabla \log h(x')|\leq L|x-x'|$. Fisher-smoothness may then be regarded as its infinite-dimensional analogue.
\begin{definition*}[Transportation-information inequality]\label{def:TI}
  Let $\cal{X}$ be a smooth Riemannian manifold. A probability measure $\gamma\in\cal{P}_p(\cal{X})$ satisfies a  transportation-information inequality with constant $\lambda$ if $\forall \varrho\in \cal{P}(\cal{X})$ such that $\mu\ll \gamma$, it holds that
  \begin{align}\label{eq:transportinfo}
    \lambda^2\,\bb{W}_p^2(\mu,\gamma)\leq\cal{I}(\mu\parallel \gamma).
  \end{align}
\end{definition*}
\begin{theorem} \label{thm:cavi}
  Let $\cal{X}^1$ and $\cal{X}^2$ be smooth Riemannian manifolds, and let $p\geq 1$. Consider $\pi=e^{-f}\in\cal{P}_p(\cal{X}^1\times\cal{X}^2)$ and assume the following.
  \begin{itemize}
    \item The optimality maps $\mu_\star^1[\cdot]:\cal{P}(\cal{X}^2)\mapsto \cal{P}(\cal{X}^1)$ and $\mu_\star^2[\cdot]:\cal{P}(\cal{X}^1)\mapsto \cal{P}(\cal{X}^2)$ are Fisher-smooth with constants $L_{12},L_{21}$ at $\ball{\mu_\star^2}{\varepsilon_{2}},\ball{\mu_\star^1}{\varepsilon_{1}}$, respectively.
    \item The fixed points $\mu_\star^1:=\mu_\star^1[\mu_\star^2[\mu_\star^1[\dots]]]$ and $\mu_\star^2:=\mu_\star^2[\mu_\star^1[\mu_\star^2[\dots]]]$ satisfy a transportation-information inequality with constants $\lambda_{1},\lambda_{2}$, respectively.
  \end{itemize}
  Suppose that $L_{12}L_{21}<\lambda_{1}\lambda_{2}$ and let $k\geq 1$. 
  Then, whenever $\mu_k^1\in \ball{\mu_\star^1}{\varepsilon_{1}}$ and $\mu_k^2 \in \ball{\mu_\star^2}{\varepsilon_{2}}$, we have
  \begin{align*}
    \bb{W}_p(\mu_{k+1}^1,\mu^1_\star) \leq \frac{L_{12}L_{21}}{\lambda_{1}\lambda_{2}} \bb{W}_p(\mu_k^1,\mu^1_\star)  \quad \text{and} \quad  \bb{W}_p(\mu_{k+1}^2,\mu_\star^2) \leq \frac{L_{12}L_{21}}{\lambda_{1}\lambda_{2}} \bb{W}_p(\mu_k^2,\mu_\star^2).
  \end{align*}
  In particular,  $(\mu_{k+1}^1,\mu_{k+1}^2)\in \ball{\mu_\star^1}{\varepsilon_{1}} \times \ball{\mu_\star^2}{\varepsilon_{2}}$ and the 2-block systematic coordinate ascent variational inference algorithm~\eqref{eq:cavi2blocks} converges exponentially fast to $(\mu_\star^1,\mu_\star^2)$ in Wasserstein-$p$ distance with rate $\smash{\tfrac{L_{12}L_{21}}{\lambda_{1}\lambda_{2}}}$.
\end{theorem}
\begin{proof}
  Since $\mu_\star^2 = \mu_\star^2[\mu_\star^1]$, $\mu_{k}^2 = \mu_\star^2[\mu^1_{k}] \in \ball{\mu_\star^2}{\varepsilon_{2}}$ and $\mu_k^1\in \ball{\mu_\star^1}{\varepsilon_{1}}$, using the Fisher-smoothness assumptions and the transportation-information inequality of $\mu_\star^2$,
  \begin{align*}
    \cal{I}(\mu_\star^1[\mu_\star^2[\mu_k^1]]\parallel\mu_\star^1) \leq L_{12}^2 \bb{W}_p^2(\mu_\star^2[\mu_k^1],\mu_\star^2) \leq\frac{L_{12}^2}{\lambda_{2}^2} \cal{I}(\mu^2_\star[\mu_k^1]\parallel\mu^2_\star) \leq\frac{L_{21}^2 L_{12}^2}{\lambda_{2}^2} \bb{W}_p^2(\mu_k^1,\mu_\star^1).
  \end{align*}
  By an application of the transportation-information inequality for $\mu_\star^1$ we find
  \begin{align*}
    \bb{W}_p^2 (\mu_\star^1[\mu_\star^2[\mu_k^1]], \mu_\star^1)  \leq \frac{1}{\lambda_{1}^2} \cal{I}(\mu_\star^1[\mu_\star^2[\mu_k^1]]\parallel\mu_\star^1) \leq \frac{L_{12}^2L_{21}^2}{\lambda_{1}^2\lambda_{2}^2} \bb{W}_p^2(\mu_k^1,\mu_\star^1).
  \end{align*}
  Since $L_{12}L_{21}<\lambda_{1}\lambda_{2}$, this shows that $\mu_{k+1}^1 = \mu_\star^1[\mu_\star^2[\mu_k^1]]\in \ball{\mu_\star^1}{\varepsilon_{1}}$, and hence the first result. Mirroring the same argument for $\mu_{k+1}^2 = \mu^2_\star[\mu^1_\star[\mu_k^2]]$ then completes the proof.
\end{proof}

Theorem~\ref{thm:cavi} characterises the Wasserstein contraction of coordinate ascent variational inference to the limiting fixed point.
If the cross-smoothness assumptions~\eqref{eq:Icrosssmooth2blocks} are verified globally, i.e.\ $\varepsilon_i = \infty$ so that $\ball{\mu_\star^i}{\varepsilon_{i}} = \cal{P}_p(\cal{X}^i)$ for both $i = 1, 2$, then the convergence is similarly global.
The convergence occurs if the small interaction condition $L_{12}L_{21}<\lambda_{1}\lambda_{2}$ holds.
This signals that the interaction terms within $\pi$ (represented by $L_{12} L_{21}$) need to be smaller than the curvature at the fixed points ($\lambda_{1} \lambda_{2}$).
For example, in the Gaussian case,
$\smash{\tfrac{L_{12}L_{21}}{\lambda_{1}\lambda_{2}}}$, coincides with the  correlation coefficient.

It is also possible to consider functional smoothness and transportation-information inequalities with more general transport distances beyond metric costs, and weak versions thereof~\citep{Gozlan2010}. We leave these extensions for future work.

\paragraph{Data augmentation.}
Often, $\pi(\dif x_1,\dif x_2)$ represents a joint posterior distribution where $x_1$ is an auxiliary random vector introduced for the sole purpose of facilitating computation, and the interest lies only in $x_2$. This framework is known as data augmentation~\citep{Tanner1987,Gelfand1990}. In these cases, the following result is helpful. It shows that a smoothness condition for the \textit{composed} optimality map $\mu_\star^2[\mu_\star^1[\cdot]]:\cal{P}(\cal{X}^2)\mapsto \cal{P}(\cal{X}^2)$ and the transportation-information inequality on $\mu_\star^2$ are all that is required to investigate the convergence of $(\mu_k^2)_{k\geq 0}$, when that is what we are interested in.

\begin{theorem} \label{thm:secondcoordinate}
  Let $\cal{X}^2$ be a smooth Riemannian manifold. For some $p\geq 1$, consider $\pi=e^{-f}\in\cal{P}_p(\cal{X}^1\times\cal{X}^2)$ and assume the following.
  \begin{itemize}
    \item  The composed optimality map $\mu_\star^2[\mu_\star^1[\cdot]]:\cal{P}(\cal{X}^2)\mapsto \cal{P}(\cal{X}^2)$ is Fisher-smooth with constant $L_2$ at $\ball{\mu_\star^2}{\varepsilon_{2}}$.
    \item The fixed point $\mu_\star^2:=\mu_\star^2[\mu_\star^1[\mu^2_\star[\dots]]]$ satisfies a transportation-information inequality with constant $\lambda_{2}$.
  \end{itemize}
  Suppose that $L_2<\lambda_{2}$ and let $k\geq 1$. Then, for all $\mu_k^2\in \ball{\mu_\star^2}{\varepsilon_{2}}$, we have
  \begin{align*}
    \bb{W}_p(\mu_{k+1}^2,\mu_\star^2) \leq \frac{L_2}{\lambda_{2}} \bb{W}_p(\mu_k^2,\mu_\star^2).
  \end{align*}
  In particular, $\mu^2_{k+1}\in \ball{\mu_\star^2}{\varepsilon_{2}}$ and $(\mu^2_k)_{k\geq 0}$ converges exponentially fast to $\mu^2_\star$ in Wasserstein-$p$ distance with rate $\smash{\tfrac{L_2}{\lambda_{2}}}$.
\end{theorem}
\begin{proof}
  Applying the transportation-information inequality for $\mu^2_\star$ and then Fisher-smoothness,
  \begin{align*}
    \lambda_{2}^2 \bb{W}_p^2(\mu_{k+1}^2,\mu^2_\star)=\lambda_{2}^2 \bb{W}_p^2(\mu_\star^2[\mu_\star^1[\mu_k^2]],\mu^2_\star) \leq \cal{I}(\mu_\star^2[\mu_\star^1[\mu^2_k]]\parallel\mu^2_\star) \leq L_2^2 \bb{W}_p^2(\mu^2_k,\mu^2_\star).
  \end{align*}
\end{proof}
By inspecting the proof of Theorem~\ref{thm:cavi}, we notice that the Fisher-smoothness of the optimality maps $\mu_\star^1[\cdot]$, $\mu_\star^2[\cdot]$ and the transportation-information inequality for $\mu_\star^1$ imply Fisher-smoothness of the composed map $\mu_\star^2[\mu_\star^1[\cdot]]$  with $L_2=L_{12}L_{21}\lambda_{1}^{-1}$.

Theorem~\ref{thm:secondcoordinate} is thus helpful when $\cal{X}^1$ involves non-smooth manifolds, or when the transportation-information inequality for $\mu_\star^1$ does not hold, or becomes complicated to verify.
We will apply this result to Gaussian Mixture Models, where $\cal{X}^1=\{0,1\}^n$ is the hypercube, and to P\'olya--Gamma augmentations in Bayesian logistic regression.
Theorem~\ref{thm:secondcoordinate} concludes exponential convergence of $(\mu_k^2)_{k\geq 0}$ without requirements on the curvature of the augmentation fixed point $\mu_\star^1$. The augmentations $x_1$ play a role only within the smoothness constant $L_2$, which captures the degree of correlation between these latent variables and $x_2$.

\subsection{Wasserstein contraction of the multi-block algorithm} \label{sec:Mblocks}

Having established contractivity of the sequential coordinate ascent variational inference in the two-block setting, we now demonstrate that the same perspective extends to the multi-block case. In this setting, we further consider the parallel and random schemes as well. This reveals links with the Gauss--Seidel and Jacobi methods for solving linear systems, and the related Stein--Rosenberg Theorem for their analysis. Our argument for coordinate ascent variational inference also suggests a novel convergence analysis for the standard coordinate descent algorithm under cross-smoothness and gradient quadratic growth conditions, which might be of independent interest; we give a more detailed discussion of this connection in Appendix~\ref{app:cd}.

The multi-block analysis naturally incurs some additional notational overhead, which we develop here.

Let $\pi = e^{-f} \in \cal{P}(\times_{i=1}^M \cal{X}^i)$ be the target distribution. Write $x = (x^1, \dots, x^M)$ for a generic point and $x^{-i} = (x^j)_{j \neq i}$ for the collection of all blocks except for $i$; write also $x^{<i} = (x^1, \dots, x^{i-1})$ and $x^{>i} = (x^{i+1}, \dots, x^M)$.  Define hence the optimality maps  $\mu^i_\star[\cdot] \colon \times_{j \neq i} \cal{P}(\cal{X}^j) \to \cal{P}(\cal{X}^i)$ by
\begin{equation}\label{eq:optimality}
	\mu^i_\star[\mu^{-i}] \propto \exp\Bigl( -\mathbb{E}_{\bigotimes_{j \neq i} \mu^j}\bigl[f(X^{<i},\cdot,X^{>i})\bigr] \Bigr).
\end{equation}
We can write down the sequential, the parallel, and the random CAVI algorithms using the optimality map as follows. The systematic scan (or sequential) CAVI processes the different blocks in cyclic order, i.e.
\begin{equation}\label{eq:sequential}
	\mu^i_{k+1} = \mu^i_\star\bigl[\mu^{<i}_{k+1}, \mu^{>i}_k\bigr], \qquad i = 1, \dots, M.
\end{equation}
The parallel CAVI instead updates all blocks simultaneously using values from the previous iteration, i.e.
\begin{equation}\label{eq:parallel}
	\mu^i_{k+1} = \mu^i_\star\bigl[\mu^{-i}_k\bigr], \qquad i = 1, \dots, M.
\end{equation}
Finally, the random-scan CAVI operates by, at each iteration, selecting a coordinate $I \in \{1,\dots,M\}$ uniformly at random, and updates only this coordinate, i.e.
\begin{align}\label{eq:random}
	\mu_{k+1}^I = \mu_\star^I[\mu_k^{-I}], \quad \mu_{k+1}^{-I}=\mu_k^{-I}, \quad I \sim \textup{Unif}(\{1,\dots,M\}).
\end{align}
The fixed-point associated to this collection of optimality maps is \begin{align*}
	\mu^i_\star = \mu^i_\star[\mu^{-i}_\star], \quad \textup{for all} \quad i=1,\dots,M,
\end{align*}
and we study the Wasserstein contraction of these various coordinate ascent variational inference schemes towards $\mu_\star=\mu_\star^1 \otimes \dots \otimes \mu_\star^M$. As in the two-block case, our assumptions will be just that each optimality map $\mu^i_\star[\cdot] \colon \times_{j \neq i} \cal{P}(\cal{X}^j) \to \cal{P}(\cal{X}^i)$ is Fisher-smooth, and each of the fixed points $\mu^i_\star\in\cal{P}(\cal{X}^i)$ satisfies a transportation-information inequality. Because in the multi-block setting the optimality maps $\mu^i_\star[\cdot] \colon \times_{j \neq i} \cal{P}(\cal{X}^j) \to \cal{P}(\cal{X}^i)$ act on a product space, we have flexibility to allow for a per-block Fisher-smoothness constant, and so we slightly generalise the previous definition.

\begin{definition*}[Fisher-smoothness]\label{def:crosssmooth}
	Let $\{\cal{X}^i\}$, $i=1,\dots,M$, be a collection of smooth Riemannian manifolds. An operator $\gamma_\star^i[\cdot]\colon \times_{j=1}^J\cal{P}(\cal{X}^j)\to\cal{P}(\cal{X}^i)$ is said to be Fisher-smooth with constants $\{L_{ij}\}_{j=1}^J$, at $\ball{\nu}{\varepsilon}=\times_{j=1}^J\ball{\nu^j}{\varepsilon_{j}} \subseteq \times_{j=1}^J \cal{P}(\cal{X}^j)$ if, for all  $\mu=\otimes_{j=1}^J \mu^j\in \ball{\nu^j}{\varepsilon_{j}}$, there holds the Lipschitz-type estimate
	\begin{align}\label{eq:Icrosssmooth}
	\cal{I}(\gamma_\star[\mu]\parallel\gamma_\star[\nu])\leq \left( \sum_{j=1}^J L_{ij} \bb{W}_p(\mu^j,{\nu^{j}}) \right)^2.
	\end{align}
\end{definition*}

We define also the vector of Wasserstein-$p$ distances of the coordinate ascent variational inference iterates at $k$ towards the fixed points as
\begin{align*}
	W_{k,p} := \bigl(\bb{W}_p(\mu^1_k, \mu^1_\star), \dots, \bb{W}_p(\mu^M_k, \mu^M_\star)\bigr)^\top \in \r^M.
\end{align*}
We also use $\preceq$ to denote element-wise partial ordering between vectors in $\mathbb{R}^M$, and
define the rate matrix $R \in \r^{M \times M}$ by
\begin{equation}\label{eq:interaction_matrix}
	R_{ij} := \frac{L_{ij}}{\lambda_i} \quad (i \neq j), \qquad R_{ii} := 0.
\end{equation}
$\rho(\cdot)$ denotes the spectral radius of a matrix. In the two-block case, one computes without difficulty that $\rho(R)^2 = \smash{\frac{L_{12}L_{21}}{\lambda_1 \lambda_2}}$, and, in analogy, we will require that $\rho(R) < 1$ here. Similarly to Theorem~\ref{thm:cavi}, we will assume that each block $i$ of the algorithm is started in a $\varepsilon_i$-ball around the fixed point, wherein the Fisher-smoothness assumption holds. In the multi-block case, however, we need to ensure that these balls are compatible with the rate matrix $R$: we assume that the Fisher-smoothness condition holds on $\varepsilon_i$-balls for some $\varepsilon_i>0$. Since $\rho(R)<1$, a small adaptation of the Perron--Frobenius Theorem guarantees that for each $\rho \in \left( \rho(R), 1 \right)$, one can identify a sub-eigenvector\footnote{In practice, one can often take $\delta$ to be an exact eigenvector of $R$, but in full generality, this requires some discussion about the sparsity pattern of $R$, which we prefer not to discuss here in details but can be verified easily on a case-by-case basis.} $\delta = (\delta_1, \dots, \delta_M) \in \r^M_+$ of $R$ with strictly positive entries such that $R\delta \prec \rho \delta$, and in particular, such that $R \delta \prec \delta$.

We now present our guarantees for the aforementioned CAVI implementations one-by-one. The proofs differ slightly from one to the next.
\begin{theorem}[Parallel CAVI] \label{thm:Dparallel}
	Let $\{\cal{X}^i\}_{i=1}^M$ be smooth Riemannian manifolds, and let $p\geq 1$. Consider $\pi=e^{-f}\in \cal{P}_p(\times_{i=1}^M \cal{X}^i)$ and assume the following.
	\begin{itemize}
		\item For all $i = 1, \dots, M$, the optimality maps $\mu^i_\star[\cdot]\colon \times_{j \neq i} \cal{P}(\cal{X}^j) \to \cal{P}(\cal{X}^i)$ are Fisher-smooth with constants $\left\{ L_{ij} \right\}_{j \neq i}$ on $\times_{j \neq i} \ball{\mu_\star^j}{\varepsilon_{j}}$.
		\item For all $i=1,\dots,M$, the fixed-point marginals $\mu_\star^i$ satisfy a transportation-information inequality with constant $\lambda_i$, respectively.
	\end{itemize}
	Suppose furthermore that $\rho(R)<1$ and consider any $c>0$ such that $c \delta_i\leq \varepsilon_i$. Then, for all $\mu_k^i\in \ball{\mu_\star^i}{c\delta_i}$, the parallel CAVI algorithm~\eqref{eq:parallel}  satisfies
	\begin{align*}
		W_{k+1,p} \preceq R W_{k,p}
	\end{align*}
	In particular,  $\mu_{k+1}^i \in \ball{\mu_\star^i}{c\delta_i}$ for all $i=1,\dots,M$, and the parallel CAVI algorithm~\eqref{eq:parallel} converges exponentially fast to $\mu_\star$ in Wasserstein-$p$ distance with rate $\rho(R)$.
\end{theorem}
\begin{proof}
	Suppose that $\mu_k^i\in \ball{\mu_\star^i}{c\delta_i}$ $i=1,\dots,M$. Then, for all $i=1,\dots,M$, by the transportation-information inequality satisfied by $\mu_\star^i$ and then Fisher-smoothness, one computes that
	\begin{align*}
		\lambda_i \bb{W}_p(\mu_{k+1}^i,\mu_\star^i) \leq \cal{I}(\mu_{k+1}^i\parallel\mu_\star^i)^{1/2} = \cal{I}(\mu_\star^i[\mu_k^{-i}]\parallel\mu_\star^i[\mu_\star^{-i}])^{1/2} \leq \sum_{j\neq i} L_{ij} \bb{W}_p(\mu_{k}^j, \mu_\star^j)
	\end{align*}
	Dividing by $\lambda_i$ shows $W_{k+1,p}\preceq R W_{k,p}$ by the definition of $R$. Moreover, $\bb{W}_p(\mu_{k+1}^i,\mu_\star^i)\leq \sum_{j\neq i}R_{ij}c\delta_j \leq c(R\delta)_i < c\delta_i$, so that $\mu_{k+1}^i\in \ball{\mu_\star^i}{c\delta_i}$.
\end{proof}
\begin{theorem}[Random-Scan CAVI] \label{thm:Drandom}
    Consider the same premises of Theorem~\ref{thm:Dparallel}. Suppose that $\rho(R)<1$, and consider any $c>0$ such that $c\delta_i\leq\varepsilon_i$ for all $i=1,\dots,M$. Then, for all $\mu_k^i\in \ball{\mu_\star^i}{c\delta_i}$, $i=1,\dots,M$, the random-scan CAVI~\eqref{eq:random} satisfies
	\begin{align*}
		\bb{E}[W_{k+1,p}] \preceq \left(\frac{1}{M}R+\frac{M-1}{M}I_M \right)W_{k,p}.
	\end{align*}
	In particular,  $\mu_{k+1}^i \in \ball{\mu_\star^i}{c\delta_i}$ for all $i=1,\dots,M$, and the random-scan CAVI algorithm~\eqref{eq:random} converges exponentially fast to $\mu_\star$ in expected Wasserstein-$p$ distance with rate $\rho(\smash{\frac{1}{M}R+\frac{M-1}{M}I_M})$.
\end{theorem}
\begin{proof}
	At iteration $k+1$, an index $I \in \{1, \dots, M\}$ is sampled uniformly at random.
	If $I = i$, the coordinate $i$ is updated as $\mu_{k+1}^i = \mu_\star^i[\mu_k^{-i}]$. As in the proof of Theorem~\ref{thm:Dparallel}, we have
	\begin{align*}
		\bb{W}_p(\mu_{k+1}^i, \mu_\star^i) \leq \sum_{j \neq i} R_{ij} \bb{W}_p(\mu_{k}^j\parallel\mu_\star^j) = (R W_{k,p})_i
	\end{align*}
	If $I \neq i$, then the coordinate is untouched and $\mu_{k+1}^i = \mu_k^i$, meaning that $\bb{W}_p(\mu_{k+1}^i, \mu_\star^i) = (W_{k,p})_i$.
	Taking the expectation with respect to the random index choice $I$ then gives
	\begin{align*}
		\mathbb{E}[\bb{W}_p(\mu_{k+1}^i, \mu_\star^i)] \leq \frac{1}{M} (R W_{k,p})_i + \frac{M-1}{M} (W_{k,p})_i,
	\end{align*}
	and vectorising this inequality across $i$ yields the result. Similarly to the parallel case, we can show $\mu_{k+1}^i\in \ball{\mu_\star^i}{c\delta_i}$.
\end{proof}
The rate admits the following interpretation: for any fixed coordinate, the random algorithm leaves it unchanged with probability $(M-1)/M$, and updates it with probability $1/M$ using the same one-coordinate update as the parallel algorithm.

For the sequential algorithm~\eqref{eq:sequential}, which updates the blocks systematically, the multi-block setting introduces additional complications, necessitating a slightly more careful inductive argument to ensure that the iterates remain in the domain of attraction of the update dynamics.

\begin{theorem}[Sequential CAVI] \label{thm:Dsequential}
    Consider the same premises of Theorem~\ref{thm:Dparallel}. Suppose that $\rho(R)<1$, and write $R=R_L+R_U$, where $R_L$ and $R_U$ are lower and upper triangular parts of $R$. Consider any $c>0$ such that $c\delta_i\leq \varepsilon_i$ for all $i=1,\dots,M$. Then, for all $\mu_k^i\in \ball{\mu_\star^i}{c\delta_i}$, $i=1,\dots,M$, the sequential CAVI~\eqref{eq:sequential} algorithm satisfies
	 \begin{align*}
	 	W_{k+1,p} \preceq (I_M-R_L)^{-1}R_U W_{k,p}.
	 \end{align*}
	 In particular,  $\mu_{k+1}^i \in \ball{\mu_\star^i}{c\delta_i}$ for all $i=1,\dots,M$, and the systematic CAVI algorithm~\eqref{eq:sequential} converges exponentially fast to $\mu_\star$ in Wasserstein-$p$ distance with rate $\rho((I_M-R_L)^{-1}R_U)$.
\end{theorem}
\begin{proof}
	We first prove by induction over the blocks $i=1,\dots,M$  that $\mu_{k+1}^i \in  \ball{\mu_\star^i}{c\delta_i}$ for all iterations. Let $i=1$.  Since $\ball{\mu_\star^j}{c\delta_j}\subseteq \ball{\mu_\star^j}{\varepsilon_{j}}$, we can use the Fisher assumption on $\times_{i=1}^M \ball{\mu_\star^i}{c\delta_i}$ to obtain
	\begin{align*}
		\lambda_1 \bb{W}_p(\mu_{k+1}^1, \mu_\star^1) \leq \cal{I}(\mu_\star^1[\mu_k^{-1}]\parallel\mu_\star^1)^{1/2} \leq \sum_{j=2}^M L_{1j}\bb{W}_p(\mu_k^{j},\mu_\star^j),
	\end{align*}
	whereby
	\begin{align*}
		 \bb{W}_p(\mu_{k+1}^1, \mu_\star^1) \leq  \sum_{j=2}^M cR_{1j}\delta_j = c(R\delta)_1 = c\rho(R)\delta_1 < c\delta_1,
	\end{align*}
	and hence implying that $\mu_{k+1}^1 \in  \ball{\mu_\star^1}{c\delta_1}$. Assume now that $(\mu_{k+1}^{j})_{j=1}^{i-1} \in  \times_{j=1}^{i-1} \ball{\mu_\star^j}{c\delta_j}$. Since $\ball{\mu_\star^j}{c\delta_j}\subseteq \ball{\mu_\star^j}{\varepsilon_{j}}$, we can use our smoothness assumption together with the induction hypothesis to obtain that
	\begin{align} \label{eq:s9bfvf}
		\lambda_i \bb{W}_p(\mu_{k+1}^i, \mu_\star^i) \leq \cal{I}(\mu_\star^i[\mu_{k+1}^{<i},\mu_{k}^{>i}]\parallel\mu_\star^i)^{1/2} \leq  \sum_{j<i} L_{ij}\bb{W}_p(\mu_{k+1}^j, \mu_\star^j) + \sum_{j>i} L_{ij}\bb{W}_p(\mu_{k}^j, \mu_\star^j).
	\end{align}
	In particular, 
	\begin{align*}
		\bb{W}_p(\mu_{k+1}^i, \mu_\star^i) \leq c \sum_{j \neq i}R_{ij}\delta_j = c(R\delta)_i = c\rho(R)\delta_i < c \delta_i 
	\end{align*}
	and the induction is complete.~\eqref{eq:s9bfvf} then holds for all $i = 1, \dots, M$, and we obtain (noting that since $R_L$ is strictly lower triangular with non-negative entries, multiplication by $(I_M-R_L)^{-1}$ is permissible) that
	\begin{align*}
		W_{k+1,p} \preceq R_L W_{k+1,p} + R_U W_{k,p} \Rightarrow W_{k+1,p} \preceq (I_M-R_L)^{-1}R_U W_{k,p}.
	\end{align*}
    The result then follows by the Stein--Rosenberg Theorem~\citep[Section 3.3]{Varga2000}, which ensures that whenever $\rho(R) < 1$, it follows that $\rho((I_M-R_L)^{-1}R_U) < 1$.
\end{proof}

If the cross-smoothness assumptions are verified globally, i.e.\ $\ball{\mu_\star^i}{\varepsilon_{i}} = \cal{P}_p(\cal{X}^i)$ and $\varepsilon_i = \infty$, then the convergence is global. Moreover, the proof also simplifies, in that no additional work is required in order to ensure that the iterates remain within the appropriate neighbourhoods.  

In practice, the spectral radius $\rho(R)$ may be conveniently bounded by use of the Gershgorin Disk Theorem, whereby one controls
\begin{align*}
    \rho\! \left( R \right) \leq \max_i \sum_{j=1}^M R_{ij}.
\end{align*}
For relatively homogeneous systems, this tends to give reasonable results; for inhomogeneous systems in which certain variables are of distinguished importance (e.g.\ star graphs), some refinements are typically required in order to obtain good estimates.

\begin{remark}[Gauss--Seidel, Jacobi, Gaussian distributions and sharpness of the rate] \label{remark:gaussian}
    Consider the coordinate ascent variational inference algorithm on the Gaussian distribution $\pi=\cal{N}(0,Q^{-1})$, $Q\in\r^{d\times d}$. For simplicity, consider $M=d$ blocks. In this case, the optimal mean-field approximation is known to be
    \begin{align*}
        \mu_\star = \bigotimes_{i=1}^M \cal{N}(0,Q_{ii}^{-1}).
    \end{align*}
    The sequential coordinate ascent variational inference iterates consists in updating the mean of Gaussians distributions as
    \begin{equation} \label{eq:cavigaussSeidel}
        \mu_{k+1}^i=\mu_\star^i[\mu^{<i}_{k+1},\mu^{>i}_k] = \cal{N}(m_{k+1}^i, Q_{ii}^{-1}), \quad m^i_{k+1}:=-Q_{ii}^{-1}\sum_{j<i}Q_{ij}m^j_{k+1}-Q_{ii}^{-1}\sum_{j>i}Q_{ij}m^j_{k}
    \end{equation}
    
    \citep[see, for instance,][Appendix A, for a derivation]{Arnese2024}, whereas the parallel algorithm computes
    \begin{equation} \label{eq:cavijacobi}
        \mu_{k+1}^i = \mu_\star^i[\mu^{-i}_{k}] = \cal{N}(m_{k+1}^i,Q_{ii}^{-1}), \quad m^i_{k+1}:=-Q_{ii}^{-1}\sum_{j\neq i}Q_{ij}m^j_{k}.
    \end{equation}
    It is immediate to compute (directly, or via Lemmas~\ref{lemma:useful} and~\ref{lemma:crosssmooth} below), that $L_{ij}=|Q_{ij}|$ and $\lambda_i=|Q_{ii}|$.
     The sequential CAVI iterates~\eqref{eq:cavigaussSeidel} correspond to a Gauss--Seidel algorithm on the means, and the parallel CAVI iterates~\eqref{eq:cavijacobi} correspond to the Jacobi algorithm. Let $R_{ij}=|Q_{ij}|/|Q_{ii}|$ and let $R_U$ and $R_L$ be its upper and lower triangular components. The Gauss--Seidel and Jacobi schemes are known to converge at rates $\rho((I-R_L)^{-1}R_U)$ and $\rho(R)$, respectively~\citep[Section 3.3]{Varga2000}. This shows that our results are sharp on Gaussian targets.
\end{remark}
\begin{remark}[Comparison of the rates and of relative speed of the algorithms] \label{remark:rosenberg}
    The Stein--Rosenberg Theorem shows that, either $\rho(R), \rho((I-R_L)^{-1}R_U) < 1$ or that $\rho(R) = \rho((I-R_L)^{-1} R_U) = 1$ both, i.e.\ that sequential, parallel and random scan CAVI algorithms are in solidarity with regard to their exponential convergence or lack thereof. Moreover, one also has $\rho((I-R_L)^{-1}R_U)<\rho(R)$, which means that the sequential CAVI algorithm is faster, in terms of progress made per-iteration, than parallel CAVI, provided that they both converge exponentially.
    
    In practice, real-time performance also depends on computational cost and hardware. Parallel CAVI's updates are naturally parallelizable, whereas systematic scan CAVI is sequential and often more storage-efficient. However, in many sparse graphical models (hidden Markov models, Markov random fields, hierarchical models, etc.), the variables can be coloured using a small number $C$ of colours so that all variables of the same colour can be updated simultaneously without conflict. Such strategies are well-known in the context of Gibbs sampling, where naive parallel implementation compromises the convergence guarantees usually associated to the method. When colouring is possible, it yields a `best-of-both-worlds' implementation of the sequential algorithm: near-parallel execution in $C$ rounds, while retaining the convergence advantages of the sequential CAVI scheme.
\end{remark}

\subsection{On the assumptions} \label{sec:hps}

We discuss some sufficient conditions for the transportation-information and Fisher-smoothness assumptions, their interpretation in the geometry of the optimal transport, their relationship with the Dobrushin criterion, and a comparison with assumptions previously considered in the literature.

\paragraph{Curvature.} The transportation-information inequalities~\eqref{eq:transportinfo} are curvature conditions on the fixed points $\mu_\star^i\in\cal{P}(\cal{X}^i)$ of the optimality operators $\mu_\star^i[\cdot]:\times_{j \neq i}\cal{P}(\cal{X}^j)\mapsto \cal{P}(\cal{X}^i)$. They are weaker than Logarithmic Sobolev inequalities (which are themselves weaker than strong log-concavity), and they tensorise, making them suitable for high-dimensional applications.
\begin{lemma}[\citet{Guillin2009}] \label{lemma:useful}
  Let $\cal{X}$ be a smooth Riemannian manifold with volume measure $\dif x$ and Ricci curvature $\Ric$, and let $\gamma=e^{-V(x)}\dif x\in\cal{P}(\cal{X})$, $V\in\cal{C}^2(\cal{X})$. The following statements are true.
  \begin{itemize}
    \item If $\gamma$ satisfies a Logarithmic Sobolev inequality with constant $\lambda$:
      \begin{align*}
        2\lambda \KL(\varrho\parallel\gamma) \leq \cal{I}(\varrho\parallel\gamma) \quad \forall \varrho\in\cal{P}(\cal{X}),
      \end{align*}
      then $\gamma$ also satisfies the transportation-information inequality for all $p\leq2$.
    \item In particular, if $\nabla^2 V + \Ric \geq \lambda I$, then the transportation-information inequality is satisfied with constant $\lambda$.
    \item Let $\{\cal{X}^i\}_{i=1}^n$ be a collection of complete connected Riemannian manifolds with $\cal{X}=\times_{i=1}^n \cal{X}^i$. If $\gamma_1,\dots,\gamma_n\in\cal{P}(\cal{X}^i)$ each satisfy a transportation-information inequality with constants $\lambda_1,\dots,\lambda_n$, respectively, then $\gamma=\otimes_{i=1}^n\gamma_i\in\cal{P}(\cal{X})$ also satisfies a transportation-information inequality, with constant $\lambda=\min_{i\in[n]}\lambda_i$.
  \end{itemize}
\end{lemma}
\begin{remark} \label{remark:inducednorms}
  When $\cal{X}=\r^d$, in the applications we will often work with the geometry induced by a positive definite matrix $B\succ 0$, namely $|v|_B^2=v^\intercal B v$. If, in this geometry, $V$ is $1$-strongly convex $\nabla^2 V\succeq B$, then $\lambda=1$ and the corresponding Fisher information uses the \textit{dual} norm
  \begin{align*}
    \cal{I}(\varrho\parallel\gamma) = \int \Big|B^{-1/2} \nabla \log \frac{\dif \varrho}{\dif \gamma}(x) \Big|^2_2 \varrho(\dif x) \quad \text{if} \quad {\varrho\ll\gamma} \quad \text{and} \quad \cal{I}(\varrho\parallel\gamma)=\infty \quad \text{otherwise}.
  \end{align*}
\end{remark}
\paragraph{Smoothness.} The Fisher-smoothness condition is a functional-level smoothness requirement on the optimality operators $\mu_\star^i[\cdot]:\times_{j \neq i}\cal{P}(\cal{X}^j)\mapsto \cal{P}(\cal{X}^i)$. 
Say $M=2$ for a moment. If we decompose the potential of the target $\pi=e^{-f}$ as
\begin{align*}
  f(x_1,x_2) =  f_1(x_1)+f_2(x_2) + \omega(x_1,x_2)
\end{align*}
for some functions $f_1:\cal{X}^1\mapsto \r$, $f_2:\cal{X}^2\mapsto \r$ and $\omega:\cal{X}^1\times\cal{X}^2\mapsto \r$, then the Fisher-smoothness of $\mu_\star^i[\cdot]$, $i=1,2$ is a requirement only on $\omega$ (which represent the interaction terms). More generally, we can always form the decomposition
\begin{align*}
	f(x_1,\dots,x_M) = \sum_{S\subseteq \{1,\dots,M\}}f_S(x_S)
\end{align*}
and only the mappings $f_S$ corresponding to subsets $S$ of $\{1, \dots, M\}$ with more than one element play a role in the Fisher-smoothness constants. In particular, only the subsets $S$ containing both $i,j$ play a role in the Lipschitz constant $L_{ij}$.

In the case $p=1$, Fisher-smoothness is related to cross-smoothness of $f$. Recall Kantorovich's duality formula: for a metric space $(\cal{X},\mathsf{d})$, the Lipschitz semi-norm of a mapping $h:\cal{X}\mapsto \r$ is
\begin{align*}
  \lip(h):=\sup\left\{\frac{|h(x_1)-h(x_2)|}{\mathsf{d}(x_1,x_2)}:x_1\neq x_2\in\cal{X} \right\},
\end{align*}
and one has the following variational characterisation of the Wasserstein-$1$ distance:
\begin{align*}
  \mathbb{W}_1(\varrho,\gamma)=\sup\left\{ \int h(x)(\varrho-\gamma)(\dif x): \lip(h)\leq 1 \right\}.
\end{align*}

\begin{lemma} \label{lemma:crosssmooth}
  Let $\pi=e^{-f}\in \cal{P}_1(\times_{i=1}^M \cal{X}^i)$. The operator $\mu_\star^i[\cdot]:\times_{j \neq i}\cal{P}(\cal{X}^j)\mapsto \cal{P}(\cal{X}^i)$  is Fisher-smooth on $\times_{j \neq i} \cal{P}_1(\cal{X}^j)$ with constants $L_{ij}:=\lip_{x_j} (\nabla_i f)$ whenever these quantities are finite.
\end{lemma}

\begin{proof}
  From the optimality conditions~\eqref{eq:optimality}, we have $\mu_\star[\mu^{-i}](x_i) \propto \exp\left(-\bb{E}_{\mu^{-i}}[f(X_{<i},x_i,X_{>i})]\right)$.
  Taking the logarithm and computing the gradient with respect to $x_i$ yields $\nabla_i \log \mu_\star[\mu^{-i}](x_i) = -\nabla_i \bb{E}_{\mu^{-i}}[f(X_{<i},x_i,X_{>i})]$.
  By the definition of $L_{ij}$, the mappings $x_j \mapsto \nabla_i f$ are $L_{ij}$-Lipschitz (uniformly in the other coordinates).
  This implies the growth bound $\|\nabla_i f(x_{<i},x_i,x_{>i})\| \leq \|\nabla_i f(x_{<i},x_i,x_{>i}')\| + \sum_{j\neq i} L_{ij}\mathsf{d}(x_j,x_j')$ for any fixed $x_{-i}'\in\cal{X}^{-i}$.
  Since $\mu^{-i} \in \cal{P}_1(\cal{X}^{-i})$, this dominating function is $\mu^{-i}$-integrable. This justifies passing the derivative inside the integral: $\nabla_i \log \mu_\star^i[\mu^{-i}](x_i) = -\bb{E}_{\mu^{-i}}[\nabla_i f(X_{<i},x_i,X_{>i})]$.
  Substituting this into the Fisher Information~\eqref{eq:fisher}, we obtain that
  \begin{align*}
    \cal{I}(\mu_\star^i[\mu^{-i}] \| \mu_\star^i[\nu^{-i}]) &=  \bb{E} \left[ \left\| \bb{E}_{\mu^{-i}}[\nabla_i f(X_{<i},X_i,X_{>i})] - \bb{E}_{\nu^{-i}}[\nabla_i f(X_{<i},X_i,X_{>i})] \right\|^2 \right] \\
    &= \bb{E} \left[ \Big\| \int \nabla_i f(x_{<i},X_i,x_{-i}) (\mu^{-i} - \nu^{-i})(\dif x_{-i}) \Big\|^2 \right] \\
    &\leq \Big(\sum_{j \neq i} L_{ij}\bb{W}_1(\mu^j,\nu^j) \Big)^2.
  \end{align*}
  The last line follows by Kantorovich's duality formula, since $\mu^{-i}$ and $\nu^{-i}$ are product measures.
\end{proof}

In other words, when $p=1$, Fisher-smoothness is related to Lipschitz assumptions on the gradient of $f$ in the opposite coordinates: it requires $\nabla_i f$ to be uniformly Lipschitz with respect to all of the remaining coordinates.

The assumptions of smoothness and curvature at the fixed points only of Theorem~\ref{thm:cavi} should be contrasted to the requirements of gradient descent-type algorithms, where one requires global smoothness and global curvature~\citep{Dalalyan2017,Durmus2017,Vempala2019,Caprio2025a,Lacker2023}.

In the coordinate ascent variational inference algorithm, we evolve an approximation of $\pi$ by iterating the optimality operators $\mu_\star^i[\cdot]$ (it is a good idea to have in mind Figure~\ref{fig:cavi2blocks}). In gradient descent type algorithms such as Langevin Monte Carlo, we evolve an approximation of $\pi$ by repeated application of one kernel, say $\cal{K}$. Somewhat analogously, in CAVI we require the operators $\mu_\star^i[\cdot]$ to be smooth, and we require curvature at the limiting points. In an algorithm such as Langevin Monte Carlo, we require the operator $\cal{K}$ to be smooth, and to exhibit curvature at the limiting point~\citep{Vempala2019}. The smoothness of the optimality operators $\mu_\star^i[\cdot]$ is related to cross-smoothness at the level of the potential $f$, whereas the smoothness of $\cal{K}$ is related to block smoothness of $f$.

\begin{remark}[Log-concavity] \label{remark:logconc}
  The transportation-information inequalities for the fixed points are weaker than Logarithmic Sobolev inequality assumptions (Lemma~\ref{lemma:useful}), which are weaker than strong log-concavity for the fixed points (the Bakry--{\'E}mery Theorem), and which are in turn weaker than the log-concavity of $\pi=e^{-f}$ when the underlying space is Euclidean. In fact, if $\nabla^2 f\succeq \lambda I_{M}$, then
  \begin{align*}
    -\nabla^2 \log \mu_\star^i(x_i)= -\nabla^2 \log \mu_\star^i[\mu^{-i}_\star](x_i) = \bb{E}_{\mu_\star^{-i}}[ \nabla^2_i f(X^{<i},x_i,X^{>i}) ] \succeq \lambda I_{M_i}
  \end{align*}
  implying $\lambda_{i}\geq \lambda$. Strong log-concavity requirements on $\pi$ are considered in many other works on CAVI~\citep{Arnese2024, Lavenant2024, Sheng2025}, and on sampling algorithms~\citep{Dalalyan2017, Durmus2017, Chewi2026}. Other than being a strict and significant relaxation of the log-concavity of $\pi$, Theorems~\ref{thm:cavi} and~\ref{thm:secondcoordinate} also allow for the possibility of multiple fixed points and analysis on more general smooth and non-smooth Riemannian manifolds.
\end{remark}

\begin{remark}[Block-smoothness and cross-smoothness]
  If $f$ is twice differentiable, then when $p=1$, the Lipschitz constants $L_{ij}$ are finite whenever the \textit{off-diagonal} elements of $f$'s Hessian are bounded, and they coincide with the respective upper bounds (Lemma~\ref{lemma:crosssmooth}). This should be contrasted with the classical block-smoothness assumptions, considered for CAVI in \citet{Arnese2024,Lavenant2024}, which are instead Lipschitz conditions on $\nabla_i f$ with respect to the same $i$th coordinate. These require, instead, the  \textit{on-diagonal} elements of $f$'s Hessian to be bounded. Neither of these implies the other.
\end{remark}
\begin{remark}[Accuracy-convergence duality] \label{remark:accconv} The  conditions under which Theorem~\ref{thm:cavi} establishes the exponential convergence of CAVI also ensure that the product variational approximation to $\pi$ is of a high quality, and vice-versa. Consider the case $\cal{X}^i=\r$ for simplicity. \citet{Lacker2024} shows that
  \begin{align*}
    0\leq  \inf_{\mu^1 \otimes \dots \otimes \mu^M}\KL(\mu^1 \otimes \dots \otimes \mu^M \parallel \pi) \leq \frac{1}{\lambda^2}\sum_{i\neq j}\bb{E}_{\mu_\star}[|\partial_{ij}f|^2]
  \end{align*}
  whenever $\pi=e^{-f}$ is $\lambda$-strongly log-concave. As we have argued in the above remarks and Lemma~\ref{lemma:crosssmooth}, $L_{ij}$ is an uniform upper bound to $|\partial^2_{ij}f|$, and when $\pi=e^{-f}$ is $\lambda$-log-concave, we always have $\lambda_{i}\geq \lambda$. Theorems~\ref{thm:Dparallel},~\ref{thm:Drandom}, or~\ref{thm:Dsequential}'s rates in the squared Wasserstein-$1$ distance can be bounded by the spectral norm-Frobenius norm inequality as
  \begin{align*}
  	\rho((I_M-R_L)^{-1}R_U)^2\leq \rho(R)^2 \leq \frac{1}{\lambda^2}\sum_{(i,j):i\neq j}L_{ij}^2
  \end{align*}
  (for the first bound, see Remark~\ref{remark:rosenberg}).
  A smaller left hand side thus implies both a fast CAVI (either sequential, parallel, or random) and a better variational approximation.
  This phenomenon has been termed the \textit{accuracy-convergence duality}~\citep{Goplerud2025} between the convergence speed of CAVI and the quality of the accuracy of the variational approximation. \citet{Goplerud2025} introduces and identifies this fact for Gaussian distributions and sequential CAVI. Theorem~\ref{thm:cavi} establishes this more generally. We will see this in action also in the applications we consider next.
\end{remark}

\begin{remark}[Wasserstein geometry]
    The assumptions under which coordinate ascent variational inference is shown to converge fast have a natural interpretation in the geometry of the Wasserstein space~\citep{Villani2009}. In this geometry, we understand the Fisher information function $\varrho\mapsto \cal{I}(\varrho \parallel \gamma)$ as the square norm of the gradient of the relative entropy $\varrho \mapsto \KL( \varrho \parallel \gamma)$. Therefore, the transport information inequalities at the fixed points are gradient quadratic growth conditions on the optimal value functions $\mu^i \mapsto \KL(\mu^i\parallel \mu_\star^i)$, $i=1,\dots,M$. On the other hand, the Fisher-smoothness conditions are contractivity assumptions on the relative entropy's gradient of the CAVI optimality maps.
    Both these observations are very natural in light of the novel proof for coordinate descent we give in Appendix~\ref{app:cd}.
\end{remark}

\subsection{Contraction on discrete spaces} \label{sec:discrete}

The preceding proofs might leave the reader forgivably thinking that this functional-analytic line of attack is rather particular to the setting of probability measures on continuous spaces. The present section will demonstrate that at a sufficient level of abstraction, the same strategy can be used to obtain results in the (a priori very different) setting of fully discrete models. In this setting, we are naturally led to consider the discrete distance $\mathsf{d}$ on $\cal{X}$ in the definition of the Wasserstein distance~\eqref{eq:wassy}. In such case, the Wasserstein distance coincides with the total variation distance, which, for two probability measures $\gamma,\varrho\in\cal{P}(\cal{X})$ is also expressible as
\begin{align*}
	\mathsf{TV} \left( \varrho,\gamma \right):= \frac{1}{2}\int \left|\frac{\dif \varrho}{\dif \gamma}(x)-1 \right| \dif \gamma(x) = \int \left( \frac{\dif \gamma}{\dif \varrho}(x)-1\right)_+ \dif \varrho(x).
\end{align*}
Because of this relationship, and because context permits, we still write $\ball{\varrho}{\varepsilon}$ for balls under the total variation metric in this section and application sections.

Let now $\mathrm{osc}$ denotes the oscillation semi-norm of a function $h:\cal{X}\mapsto \r$:
\begin{align*}
	\mathrm{osc} (h ) := \sup \{ h(x) -h (x^\prime) : x, x^\prime \in \mathcal{X} \}.
\end{align*}
The transportation-information inequality for discrete measures is meant in the following sense.
\begin{definition*}[Discrete Transportation-Information Inequality]\label{def:disc-TI}
	Let $\cal{X}$ be a discrete space. A probability measure $\gamma \in \mathcal{P} \left( \mathcal{X} \right)$ satisfies a discrete transportation-information inequality with constant $\lambda > 0$ if for all $\mu \ll \gamma$, there holds the inequality
	\begin{align*}
		\lambda \mathsf{TV} \left( \mu, \gamma \right) \leq \mathrm{osc} \left( \log \frac{\mathrm{d} \mu}{\mathrm{d} \gamma} \right).
	\end{align*}
\end{definition*}

As opposed to the continuous case,  \emph{any} probability measure satisfies a discrete transportation-information inequality with constant at least $\lambda = 4$. This should be compared with the analogous result for transportation-entropy inequalities with respect to the discrete distance, i.e.\ Pinsker's inequality; see Appendix~\ref{app:dfi} for some informal discussion on these points.

\begin{lemma} \label{lemma:usefuldiscrete}
	Let $\cal{X}$ be a countable space. Then, any probability measure $\gamma\in\cal{P}(\cal{X})$ satisfies a discrete transportation-information inequality, with constant at least $\lambda=4$. 
\end{lemma}
\begin{proof}
	Set $h:=\smash{\frac{\dif \gamma}{\dif \varrho}}$ and define $m:=\inf h$, $M:=\sup h$. Since $\int h \dif \varrho=1$, we have $0<m\leq 1 \leq M$. Since $\mathsf{TV}(\varrho,\gamma)=\bb{E}_\varrho(h-1)_+$, for all $x\in[m,M]$, 
	\begin{align*}
		(x-1)_+ \leq \frac{M-1}{M-m}(x-m).
	\end{align*}
	Taking expectations with respect  to $\varrho$, we obtain
	\begin{align*}
		\mathsf{TV}(\varrho,\gamma) \leq \frac{M-1}{M-m}(1-m),
	\end{align*}
	Let $s:=0.5(\log M-\log m).$ and $c=0.5(\log M+\log m)$. Since $|c|\leq s$, argue that
	\begin{align*}
		\frac{M-1}{M-m}(1-m) = \frac{\cosh s - \cosh c }{\sinh s} \leq \frac{\cosh s- 1}{\sinh s} = \tanh (s/2) = \tanh\left(\frac{1}{4}\log\frac{M}{m}\right);
	\end{align*}
	the result then follows by definition of oscillation semi-norm and the fact that $\tanh(x)\leq x$.
\end{proof}
In particular, the constant $\lambda=4$ is achieved by the law of an uniform two-points Bernoulli distribution. Discrete Fisher-smoothness then follows in analogy with the continuous case, and it is similarly ensured under a boundedness condition on the cross-terms of the `Hessian' of the potentials.
\begin{definition*}[Discrete Fisher-Smoothness]\label{def:disc-smooth}
	Let $\{\cal{X}^i\}$, $i=1,\dots,M$ be a collection of discrete spaces.
	An operator $\gamma_\star^i[\cdot]\colon \times_{j=1}^J \cal{P}(\cal{X}^j) \to \cal{P}(\cal{X}^i)$ is said to be discretely Fisher-smooth with constants $\{L_{ij}\}_{j=1}^J$, at $\ball{\nu}{\varepsilon}=\times_{j=1}^J\ball{\nu^j}{\varepsilon_{j}} \subseteq \times_{j=1}^J \cal{P}(\cal{X}^j)$ if, for all  $\mu=\otimes_{j=1}^J \mu^j\in \ball{\nu^j}{\varepsilon_{j}}$, there holds the Lipschitz-type estimate
	\begin{align*}
		\mathrm{osc} \left( \log \frac{\mathrm{d} \gamma_\star^i \left[ \mu \right]}{\mathrm{d} \gamma_\star^i \left[ \nu \right]} \right) \leq \sum_{j\neq i} L_{ij} \mathsf{TV} ( \mu^j, \nu^j ).
	\end{align*}
\end{definition*}
\begin{definition*}[Discrete Lipschitz Conditions]\label{def:disc-lips}
	Given a mapping $f: \times_{i = 1}^M \mathcal{X}^i\to\mathbb{R}$,
	\begin{enumerate}
\item For $1 \leq j \leq M$, define the $j$th coordinate-wise Lipschitz constant of $f$ by
        \begin{align*}
            \delta_{j}^{\left(1\right)} \left(f\right) &= \sup_{x^{-j}}\sup_{x^j_1, x^j_2} \Delta^{(1)} \\
            \Delta^{(1)} &= f\left(x^{-j}, x^j_1 \right)-f\left(x^{-j},   x^j_2\right).
        \end{align*}
        \item For $1 \leq i,j \leq M$, define the $(i,j)$th cross-smoothness constant of $f$ by
        \begin{align*}
            \delta_{i,j}^{\left(2\right)} \left(f\right) &= \sup_{x^{-\{i, j  \}}}\sup_{x^i_1, x^i_2, x^j_1, x^j_2} \Delta^{(2)} \\
			\Delta^{(2)} &= f\left(x^{-\{i, j  \}}, x^i_1, x^j_1  \right) + f\left(x^{-\{i, j  \}}, x^i_2, x^j_2\right) \\
			&\qquad -f\left(x^{-\{i, j  \}}, x^i_1, x^j_2\right)-f\left(x^{-\{i, j  \}}, x^i_2, x^j_1 \right),
		\end{align*}
	\end{enumerate}
\end{definition*}
These should be thought of as analogous to bounds on the gradients and mixed second-order partial derivatives of functions on smooth spaces. From this perspective, it is not so surprising that a sort of discrete Clairaut--Schwarz--Young Theorem persists, whereby the `order of differentiation' is ignorable.
\begin{lemma}
	There holds the equality
	\begin{align*}
		\delta_{jj^\prime }^{\left(2\right)} \left(f\right) = \sup_{x^i_1, x^i_2} \delta_{j^\prime}^{(1)} \left( f \left( x^i_1,\cdot \right) - f \left( x^i_2,\cdot \right) \right),
	\end{align*}
	where $f\left( x^j,\cdot \right) : \prod_{1 \leq j^\prime \leq J, j^\prime \neq j} \mathcal{X}^{j^\prime} \to \mathbb{R}$ maps $x^{-j} = \left\{ x^k : k \neq j \right\}$ to $f(x^j, x^{-j})$.
\end{lemma}
\begin{lemma} \label{lemma:discretecrosssmooth}
	Let $\pi=e^{-f}\in \cal{P}_1(\times_{i=1}^M \cal{X}^i)$. The operator $\mu_\star^i[\cdot]:\times_{j \neq i}\cal{P}(\cal{X}^j)\mapsto \cal{P}(\cal{X}^i)$  is discretely Fisher-smooth on $\times_{j \neq i} \cal{P}_1(\cal{X}^j)$ with constants $L_{ij}:=\delta^{(2)}_{ij}(f)$ whenever these quantities are finite.
\end{lemma}
We now present a simple contraction estimate for CAVI under these discretised assumptions. To state our result, we define the vector of marginal total variation distances of the coordinate ascent variational inference iterates at $k$ towards the fixed points as
\begin{align*}
	T_k := \bigl( \mathsf{TV}(\mu^1_k, \mu^1_\star), \dots, \mathsf{TV}(\mu^M_k, \mu^M_\star)\bigr)^\top \in \r^M.
\end{align*}
Analogous guarantees for parallel CAVI, random-scan CAVI, and sequential CAVI follow from the same steps outlined in Section 2.4; we state the corresponding results without proof.
\begin{theorem} \label{thm:cavidiscrete}
	Let $\left\{ \mathcal{X}^i \right\}_{i = 1}^M$ be discrete spaces. Consider $\pi = e^{-f} \in \mathcal{P} (\times_{i=1}^M \mathcal{X}^i )$ and define $\ball{\mu_\star^i}{\varepsilon_{i}} = \{ \mu^i \in \mathcal{P} ( \mathcal{X}^i ) : \mathsf{TV} ( \mu^i, \mu_\star^i ) < \varepsilon_i \}$. Assume the following.
	\begin{itemize}
		\item For all $i = 1, \ldots, M$, the optimality map $\mu_\star^i \left[ \cdot \right]$ is discretely Fisher-smooth with constants $L_{ij}$ on $\times_{j \neq i} \ball{\mu_\star^j}{\varepsilon_{j}}$.
		\item For all $i = 1, \ldots, M$, the fixed-point marginal $\mu_\star^i$ satisfies a discrete transportation-information inequality with constant $\lambda_i$.
	\end{itemize}
	Define the discrete rate matrix $R\in\r^{M\times M}$ as $R_{ij}=L_{ij}/\lambda_i$ for $i\neq j$ and $R_{ii}=0$. Suppose furthermore that $\rho(R)<1$, let $\delta\in(0,\infty)^M$ be such that $R\delta\prec\delta$, and consider any $c>0$ such that $c\delta_i \leq \varepsilon_i$ for all $i=1,\dots,M$. Then, 
	\begin{enumerate}
		\item For all $\mu_k \in \times_{i=1}^M \ball{\mu_\star^i}{c\delta_i}$, the parallel CAVI~\eqref{eq:parallel} satisfies
		\begin{align*}
			T_{k+1} \preceq R T_{k}
		\end{align*}
		In particular, $\mu_{k+1} \in  \times_{i=1}^M \ball{\mu_\star^i}{c\delta_i}$ and the parallel CAVI algorithm~\eqref{eq:parallel} converges exponentially fast to $\mu_\star$ in total variation distance with rate $\rho(R)$.
		\item For all $\mu_k \in \times_{i=1}^M \ball{\mu_\star^i}{c\delta_i}$, the random-scan CAVI algorithm~\eqref{eq:random} satisfies
		\begin{align*}
			\mathbb{E}[T_{k+1}] \preceq \left( \frac{1}{M} R + \frac{M-1}{M} I_M\right) T_k
		\end{align*}
		In particular, $\mu_{k+1} \in  \times_{i=1}^M \ball{\mu_\star^i}{c\delta_i}$ and the random-scan CAVI algorithm~\eqref{eq:random} converges exponentially fast to $\mu_\star$ in expected total variation distance with rate $\rho(\smash{\frac{1}{M}R+\frac{M-1}{M}I})$.
		\item Writing $R = R_L + R_U$, where $R_L$ and $R_U$ are the lower- and upper-triangular parts of $R$, for all $\mu_k\in \times_{i=1}^M \ball{\mu_\star^i}{c\delta_i}$, the sequential CAVI~\eqref{eq:sequential} algorithm satisfies
		\begin{align*}
			T_{k+1} \preceq (I_M-R_L)^{-1}R_U T_k.
		\end{align*}
		In particular,  $\mu_{k+1} \in  \times_{i=1}^M \ball{\mu_\star^i}{c\delta_i}$ and the sequential CAVI algorithm~\eqref{eq:sequential} converges exponentially fast to $\mu_\star$ in total variation distance with rate $\rho((I_M-R_L)^{-1}R_U)$.
	\end{enumerate}
\end{theorem}

\section{Applications to statistical physics and Bayesian models} \label{sec:appls}

\subsection{Bayesian Gaussian Mixtures} \label{sec:gmm}
We consider the unbalanced two-component Bayesian Gaussian Mixture Model
\begin{align} \label{eq:gmm}
  Y_i \sim p\cal{N}(\beta,\tau^{-1}I_d) + (1-p)\cal{N}(-\beta,\tau^{-1}I_d), \quad i=1,\dots, n,
\end{align}
for some given $p\in(0,1)$. The latent auxiliary variables  $Z_i\in\{0,1\}$ indicate the mixture components, i.e.\ $Y_i|Z_i=1 \sim \cal{N}(\beta,\tau^{-1}I_d)$ and $Y_i|Z_i=0 \sim \cal{N}(-\beta,\tau^{-1}I_d)$. We posit a Gaussian prior on the mean $\beta\sim\cal{N}(0;\tau_0^{-1}I_d)$. We focus on this class of simple models to facilitate comparison with some related results which are available for the expectation-maximisation (EM) algorithm, and to more cleanly illustrate the arguments.

Let $z=(z_1,\dots,z_n)$, $y=(y_1,\dots,y_n)$. We would like to approximate the intractable posterior
\begin{align*}
  \pi(z,\beta) = \prod_{i=1}^n \left\{ p\delta_{\{z_i=1\}}\cal{N}(y_i;\beta,\tau^{-1}I_d) + (1-p) \delta_{\{z_i=0\}}\cal{N}(y_i;-\beta,\tau^{-1}I_d) \right\}\cal{N}(\beta,0,\tau_0^{-1}I_d)
\end{align*}
with a product distribution of the form $\mu^1(\dif z)\otimes\mu^2(\dif \beta)$ to be fit via the coordinate ascent variational inference algorithm~\citep[Chapter 10]{Bishop2006}.
The CAVI updates for $\mu$ consist in updating Bernoulli allocation probabilities and the mean of the Gaussian distribution as
\begin{equation} \label{eq:cavigmm}
  \begin{aligned}
    \mu_\star^1[\mu^2] &= \bigotimes_{i=1}^n \Ber(r_{i,\mu^2}), \quad
    \mu_\star^2[\mu^1] = \cal{N}(m_{\mu^1},(\tau_0 + n\tau)^{-1} I_d),
  \end{aligned}
\end{equation}
where $\logit(r_{i,\mu^2}) = \log (p/(1-p)) + 2\tau y_i^\intercal \mathbb{E}_{\mu^2} [\beta] $ and $m_{\mu^1}:=(\tau_0 + n\tau)^{-1}\tau\sum_{i=1}^n \bb{E}_{\mu^1}[2Z_i-1]y_i$.
To analyse CAVI's exponential convergence for this class of models, we aim to apply Theorem~\ref{thm:secondcoordinate} with $p=2$, noting that our interest only lies in the component $\beta$. We emphasise that the auxiliary variables $z$ are only introduced for computational reasons, and are of only secondary interest inferentially.

\begin{proposition}
  Suppose that the starting distribution satisfies $\mu_0^2\in \{\mu^2\in \cal{P}_2(\cal{B}):\bb{W}_2(\mu^2,\mu_{\star}^2)<\varepsilon\}$. Then, $(\mu^2_k)$ converges exponentially fast in Wasserstein-$2$ distance with rate $r_\varepsilon$ if
  \begin{align} \label{eq:gmmlocalcond}
    r_{\varepsilon}:=\frac{\tau^2}{\tau_0+\tau n} \sup_{|x-\bb{E}_{\mu_\star^2}[\beta]|\leq \varepsilon} \lambda_{\max}\Bigg(\sum_{i=1}^n \sech^2\Big(\frac{1}{2}\log \frac{p}{1-p}+\tau y_i^\intercal x\Big) y_i y_i^\intercal \Bigg) < 1.
  \end{align}
\end{proposition}
\begin{proof}
  We apply Theorem~\ref{thm:secondcoordinate}. We readily have $\lambda_{2}=n\tau+\tau_0$ (Lemma~\ref{lemma:useful}), so we only need to find $L_2$. Apply the formula for the Fisher information between Gaussians:
  \begin{align*}
    \cal{I}(\mu_\star^2[\mu_\star^1[\mu^2]]\parallel\mu^2_\star) = (\tau n + \tau_0)^2 |m_{\mu_\star^1[\mu^2]}-m_{\mu_\star^1}|^2
  \end{align*}
  Since $r_{i,\mu^2}=\sigma(\log(p/(1-p))+2\tau y_i^\intercal \bb{E}_{\mu^2}[\beta])$ and $\tanh(x)=2\sigma(2x)-1$, we have $2r_i[\mu^2]-1 = \tanh(0.5\log(p/(1-p))+\tau y_i^\intercal \bb{E}_{\mu^2}[\beta])$, hence
  \begin{align*}
    \cal{I}(\mu_\star^2[\mu_\star^1[\mu^2]]\parallel\mu_\star^2) = |\varphi(\bb{E}_{\mu^2}[\beta])-\varphi(\bb{E}_{\mu_\star^2}[\beta]|^2, \quad \varphi(x):=\tau \sum_{i=1}^n \tanh\Big(\frac{1}{2}\log \frac{p}{1-p}+ \tau y_i^\intercal x\Big)y_i\in\r^d.
\end{align*}
Computing the Jacobian of $\varphi$ as
\begin{align*}
    J\varphi(x)=\tau^2 \sum_{i=1}^n\sech^2\Big(\frac{1}{2}\log \frac{p}{1-p}+\tau y_i^\intercal x\Big)y_i y_i^\intercal
\end{align*}
and recalling that the Wasserstein-2 distance controls the means difference $|\bb{E}_{\mu^2}[\beta]-\bb{E}_{\mu_\star^2}[\beta]|\leq \bb{W}_2(\mu^2,\mu_\star^2)$, the result follows by a direct application of the mean value inequality.
\end{proof}

\noindent
With these results in hand, we can make the following observations. Consider first the balanced case $p=1/2$ for simplicity.
\begin{itemize}[leftmargin=.25cm]
\item Suppose that $\varepsilon \to \infty$, i.e.\ that we seek to characterise the global convergence of coordinate ascent variational inference. In this case, since $0<\sech^2(\tau y^\intercal x)\leq 1$, we obtain
    \begin{align*}
    r_\varepsilon \to r := \frac{\tau^2}{\tau_0+\tau n}  \lambda_{\max}\Big(\sum_{i=1}^n y_i y_i^\intercal \Big)
    \end{align*}
    and CAVI converges globally if $r<1$.
    This is a posterior phase transition criterion. Since the posterior satisfies $p(y_i|\beta)\propto e^{-\frac{\tau}{2}(|y_i|^2+|\beta|^2)}\cosh(\tau y_i^\intercal \beta)$, the log-posterior and its Hessian at $\beta=0$ reads
    \begin{align*}
    \log p(\beta|y) = -\frac{n\tau+\tau_0}{2}|\beta|^2 + \sum_{i=1}^n \log \cosh(\tau y_i^\intercal \beta), \quad \nabla^2 \log p(\beta|y) = - (n\tau + \tau_0)I_d + \tau^2 \sum_{i=1}^n y_i y_i^\intercal
    \end{align*}
    For $\beta=0$ to be a strict local maximum, one requires indeed $r<1$, in which case the posterior is also unimodal. In this case, the posterior landscape is very well-behaved, and CAVI is globally exponentially convergent, but the mixture elements are completely unidentifiable. If the model is well-specified, then this happens very rarely as the number of data points $n$ grows: letting $\beta_0$ denote the true population parameter, in the data-rich limit, the law of large numbers ensures that
    \begin{align*}
    r = \lambda_{\max} \Big(\frac{\tau^2}{\tau_0+\tau n} \sum_{i=1}^n y_i y_i^\intercal \Big) \to \lambda_{\max}(I_d+\tau \beta_0 \beta^\intercal_0) = 1+ \tau \lambda_{\max}(\beta_0 {\beta_0}^\intercal)  >1,
    \end{align*}
    almost-surely. As such, global convergence can only be expected with a large prior precision $\tau_0$ at zero, with few data points, or if the model is misspecified.
\item When $r>1$, the symmetric stationary point $\beta=0$ is not a strict local maximum. By symmetry, since the posterior is coercive, there must be a pair of symmetric maximisers. The variational posterior $\nu_\star$, however, can only be unimodal, and different initialisations can converge to different Gaussian approximations, each centred around a different posterior mode.
\item Suppose now that $\varepsilon< |\bb{E}_{\mu_\star^2}[\beta]|$, that is, the variational posterior mean (or the posterior mode separation) is large, relative to our initial displacement from the optimum. In this case, the $\sech$ term is active, and it will force $r_\varepsilon<1$ for a sufficiently large posterior mode separation  $|\bb{E}_{\mu_\star^2}[\beta]|$ or a sufficiently small $\varepsilon$, leading to local exponential convergence. As the posterior mode separation grows, the initialisation can be taken further away. Moreover, the convergence speed of CAVI depends positively on the magnitude of the posterior modes separation, and starting closer yields provably faster convergence. When the precision $\tau$ is large, the gap $|\bb{E}_{\mu_\star^2}[\beta]|-\varepsilon$ may be larger.
\item Suppose that the ball $\{x:|x-\bb{E}_{\mu_\star^2}[\beta]|\leq \varepsilon\}$ contains zero, i.e.\ that ~$\varepsilon\geq |\bb{E}_{\mu_\star^2}[\beta]|$. This corresponds to a situation where we start the algorithm in a neighbourhood of the posterior mode that is so large it contains the saddle point at zero. In this case,  $r_\varepsilon = r$, and fast convergence is possible only when it also happens globally.
\end{itemize}

In the unbalanced general case $p\in (0,1)$, the argument within the $\sech$ will typically be at least as large as in the $p=1/2$ case. Because $\sech$ decreases away from the origin, coordinate ascent variational inference can be expected to converge more rapidly in the asymmetric case, and the radius of exponential convergence will be wider. Unbalanced mixtures make the labels less ambiguous, facilitate inference, and typically lead to a faster Bayesian computation. Figure~\ref{fig:cavigmm} illustrates these phenomena numerically.

These types of results seem to be novel for coordinate ascent variational inference, and elicit parallels with some of the convergence theory of the expectation-maximisation (EM) algorithm for Gaussian Mixture Models~\citep{Xu1996,Balakrishnan2017,Daskalakis2017,Weinberger2022}. In this context, $r$ behaves as an empirical, Bayesian, signal-to-noise ratio, and the variational posterior mean $\bb{E}_{\mu_\star^2}[\beta]$ plays the role of the maximum likelihood estimator.

\begin{figure}[!tbp]
\caption{Log-Wasserstein-$2$ convergence of CAVI in Gaussian Mixture Models ($d=10$, $n=200$, averaged across 100 simulations). \textbf{Left Panel:} Local exponential convergence across varying mean separations $|\bb{E}_{\nu_\star}[\beta]|$ for a fixed $\varepsilon$ and precisions $\tau=\tau_0=0.1$. Greater mean separation accelerates convergence and induces an earlier onset of the exponential phase (characterised by linear trajectories). \textbf{Right Panel:} Local exponential convergence across varying precisions $\tau$ for a fixed mean separation and $\varepsilon$. Higher precision similarly yields faster convergence and an earlier transition to the exponential phase.}
\label{fig:cavigmm}
\centering
\begin{minipage}[b]{0.45\textwidth}
   \includegraphics[width=\textwidth]{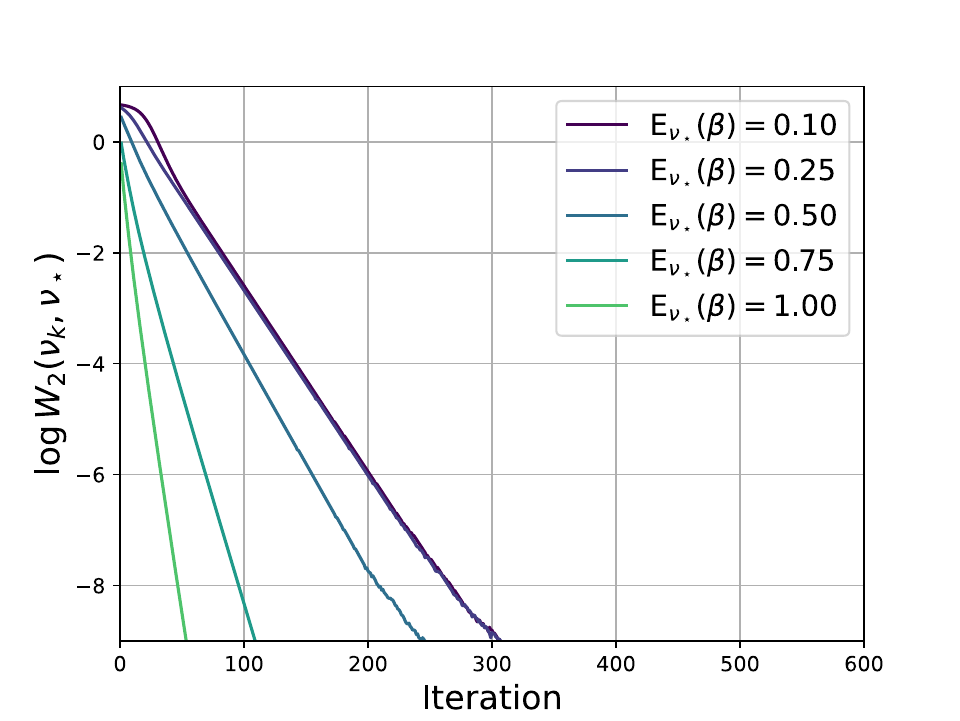}
\end{minipage}
\hfill
\begin{minipage}[b]{0.45\textwidth}
   \includegraphics[width=\textwidth]{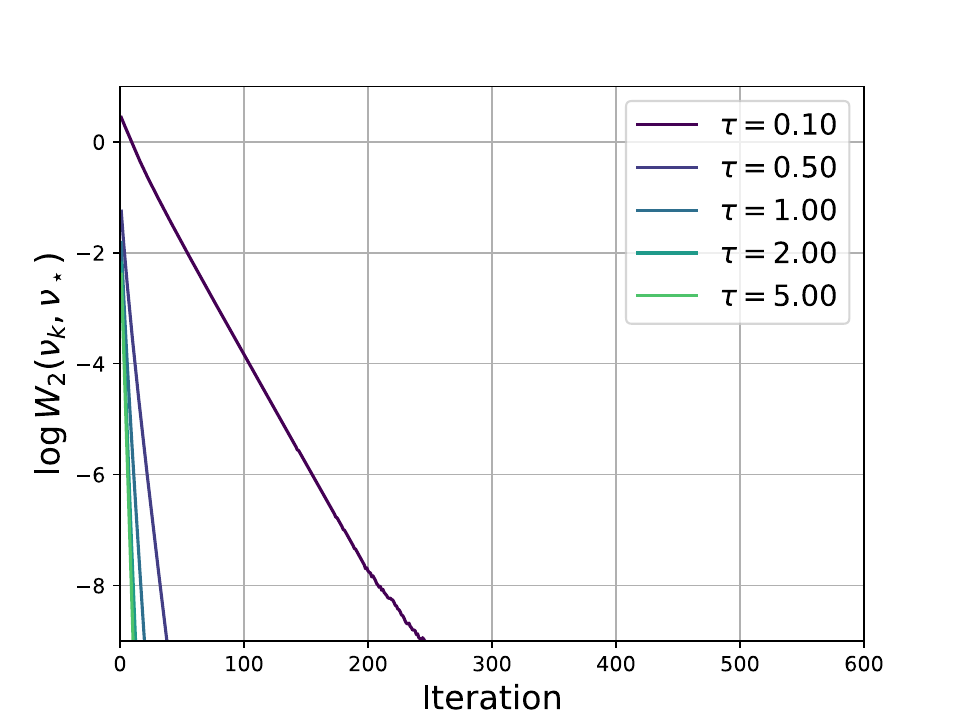}
\end{minipage}
\end{figure}

\subsection{Bayesian Classification} \label{sec:classi}
We assume that the observations are generated according to $$Y_i | \beta \sim \Ber(\Phi(x_i^\intercal \beta)), \quad i=1,\dots,n,$$ where $\Phi$ represents a link function of choice, and with $x_i$ indicating the $i$-th row of the design matrix $X\in\r^{n\times p}$. We posit a standard Gaussian prior $\beta \sim \cal{N}(m_0,Q_0^{-1})$, with $Q_0\in\r^{p\times p}$ potentially depending on the design.  Denote $y=(y_1,\dots,y_n)$. The posterior distribution
\begin{align}
\label{eq:bayclassposterior}
\pi(\dif \beta|y) = \cal{N}(\dif \beta;m_0,Q_0^{-1}) \prod_{i=1}^n \Phi(x_i^\intercal \beta)^{y_i}(1-\Phi(x_i^\intercal \beta))^{1-y_i}
\end{align}
of interest is typically intractable, and hard to approximate directly. A common strategy in Bayesian computation is data augmentation~\citep{Tanner1987,Gelfand1990}, which introduces a collection of random variables $(Z_i)$ and an augmented posterior distribution $\pi(\dif z,\dif \beta)$ admitting~\eqref{eq:bayclassposterior} as marginal, and that allow for standard Bayesian computation algorithms to be implemented directly. The choice of distribution of the auxiliary random variables $(Z_i)$ used for augmentation depends on the link function. We consider two common choices.

\paragraph{Probit Regression.} \label{subsec:probit}
Suppose that $\Phi$ is the cumulative distribution function of a Gaussian distribution. In this case, we may consider augmenting our parameter space with latent variables $Z_i\sim \cal{N}(x_i^\intercal \beta, 1)$. This allows us to exhibit our posterior of interest as the $\beta$-marginal of
\begin{align*}
\pi(\dif z, \dif \beta) = \cal{N}(\dif \beta;m_0,Q_0^{-1})  \cal{N}(\dif z;X\beta,I_n) \prod_{i=1}^n \{y_i\mathbf{1}{(z_i>0)} + (1-y_i)\mathbf{1}{(z_i\leq 0)} \}
\end{align*}
\citep{Albert1993}.
We consider a two-block coordinate ascent variational approach to approximate the augmented posterior. The CAVI updates are available in closed-form, and are given concretely by a Gaussian and a product-of-truncated-Gaussians, with explicit forms
\begin{equation} \label{eq:cavibayesprobit}
\begin{aligned}
    \mu_\star^1[\mu^2](\dif z)&\propto  \bigotimes_{i=1}^n \cal{N}(\dif z_i;x_i^\intercal \mathbb{E}_{\mu^2} [\beta],1)\{y_i\mathbf{1}{(z_i>0)} + (1-y_i)\mathbf{1}{(z_i\leq 0)} \}, \\
    \mu_\star^2[\mu^1](\dif \beta) &=\cal{N}(\dif \beta;m_{\mu^1},Q^{-1}),
\end{aligned}
\end{equation}
where $ m_{\mu^1}=Q^{-1}(Q_0 m_0 +X^\intercal \mathbb{E}_{\mu^1} [Z]), \quad Q=Q_0+X^\intercal X$.
We characterise the exponential convergence of CAVI in Bayesian Probit Regression problems~\eqref{eq:cavibayesprobit} in terms of the \textit{Bayesian fraction of missing information}
\begin{align*}
r:=\lambda_{\max}((Q_0+X^\intercal X)^{-1}X^\intercal X).
\end{align*}
\begin{proposition}    \label{cor:cavibayesprobit}
The iterates $(\mu_k^1,\mu_k^2)_{k\geq 0}$ from the CAVI algorithm~\eqref{eq:cavibayesprobit} for Bayesian Probit Regression converge globally exponentially fast in Wasserstein-$2$ distance with rate $r=\lambda_{\max}((Q_0+X^\intercal X)^{-1}X^\intercal X)$.
\end{proposition}
\begin{proof}
We aim to apply Theorem~\ref{thm:cavi} with $p=2$. It is convenient to work with the norm induced by $Q=Q_0+X^\intercal X$ on $\cal{B}$ and the standard Euclidean on $\cal{Z}$ (have a look at Remark~\ref{remark:inducednorms}). Recalling that the Logarithmic Sobolev inequality implies the transportation-information inequality (Lemma~\ref{lemma:useful}) and that the transportation-information inequality tensorises, the strong log-concavity of $\mu_\star^1[\mu^2]$ and $\mu_\star^2[\mu^1]$ yields that $\lambda_{1} = \lambda_{2}=1$ in this norm, so we only need to find $L_{12}$ and $L_{21}$. Let $A:=(Q_0+X^\intercal X)^{-1/2}X^\intercal$, and notice that $\lambda_{\max}((Q_0+X^\intercal X)^{-1}X^\intercal X)=\lambda_{\max}(A^\intercal A)$. The Fisher information functional between two Gaussians
$\varrho := \cal{N}(b,B^{-1})$ and $\gamma:= \cal{N}(c,B^{-1})$ with the same covariance is equal to $\cal{I}(\varrho\|\gamma) = \|B^{1/2}(b - c)\|^2_2$.
In particular, since $\mu_\star^2[\mu^1]$ and $\mu_\star^2[\nu^1]$ have the same covariance for arbitrary $\mu^1$ and $\nu^1$, in the norm induced by $Q$,
\begin{align*}
 \cal{I}(\mu_\star^2[\mu^1] \parallel \mu_\star^2[\nu^1]) = \|Q^{-1/2}QQ^{-1}X^\intercal( \bb{E}_{\mu^1}[Z]-\bb{E}_{\nu^1}[Z])\|^2_2 \leq \lambda_{\max}(A^\intercal A) \|\bb{E}_{\mu^1}[Z]-\bb{E}_{\nu^1}[Z]\|^2_2
\end{align*}
Since $\|\bb{E}_{\mu^1}[Z]-\bb{E}_{\nu^1}[Z]\|^2_2\leq \bb{W}_2^2(\mu^1,\nu^1)$, it follows that $L_{12}=\lambda_{\max}^{1/2}(A^\intercal A)$, and similarly for $L_{21}$. Theorem~\ref{thm:cavi} gives a convergence rate in Wasserstein-$2$ distance equal to $\lambda_{\max}(A^\intercal A)$.
\end{proof}

The Bayesian fraction of missing information is a ratio of posterior certainty provided by the data $X^\intercal X$ on the total $Q_0+X^\intercal X$. In particular, if there is little prior information and $Q_0\to 0$ in a suitable sense, then all of the information comes from the data, and CAVI converges slowly as a consequence. Equally, if $X^\intercal X$ is large in a suitable sense, then the data is very informative, and CAVI is slow. We remark that for this specific example, the results of \citet{Bhattacharya2025} are also applicable, and return the same convergence rate.

In the random design setting, following the approach in \citet{Lee2024,Ascolani2025} for Gibbs sampling, we can use results from random matrix theory to show the rate stays bounded away from $1$ almost surely, even as $n, p\to \infty$, provided that we scale the prior precision $Q_0$ appropriately with $p$ (or, alternatively, the design matrix). We consider two choices of importance in Bayesian statistics: (i) a direct scaling with $p$, $Q_0 = (p/c) I_p$ for some $c > 0$, common in the context of high-dimensional Bayesian regression~\citep{Simpson2017}, and (ii) the generalised g-prior $Q_0=(X^\intercal X/g + c I_{p})$~\citep{Zellner1986,Ascolani2025}, again for some $c > 0$.
\begin{corollary} \label{cor:cavibayesprobitrates} The following statements hold true.
\begin{itemize}
    \item Let $Q_0=(p/c)I_p$. Then,
    \begin{align*}
        r =\frac{(c/p)\lambda_{\max}(X^\intercal X)}{1+(c/p)\lambda_{\max}(X^\intercal X)} < 1.
    \end{align*}
    In particular, suppose that $X_{i,j}$ are i.i.d.~with unit variance. If $n,p\to \infty$ and $n/p\to a\in(0,\infty)$, we have
    \begin{align*}
        \lim r = \frac{c(1+\sqrt{a})^2}{1+c(1+\sqrt{a})^2} < 1
    \end{align*}
    \item  Let $Q_0=(X^\intercal X/g + c I_{p})$. Then,
    \begin{align*}
        r =  \frac{\lambda_{\max}(X^\intercal X)}{(1+1/g)\lambda_{\max}(X^\intercal X)+c} < 1.
    \end{align*}
    In particular, suppose that $X_{i,j}$ are i.i.d.~with unit variance. If $n,p\to \infty$ and $n/p\to a\in(0,\infty)$, we have
    \begin{align*}
        \lim r = \frac{1}{(1+1/g)} < 1
    \end{align*}
\end{itemize}
\end{corollary}
\begin{proof}
Let $Q_0=(p/c)I_d$ and define $h(A):=((p/c)I_d+A)^{-1}A$. By the spectral mapping theorem~\citep[Theorem 1.13]{Higham2008},
\begin{align} \label{eq:sad89ahdsbfvf}
    r = \lambda_{\max}(h(X^\intercal X)) = \max_{\lambda\in\sigma(X^\intercal X)}h(\lambda),
\end{align}
where $\sigma(A)$ denotes the spectrum of $A$. Because $t\mapsto h(t)$ is increasing, we have
\begin{align*}
    r =  \frac{\lambda_{\max}(X^\intercal X)}{(p/c)+\lambda_{\max}(X^\intercal X)}
\end{align*}
and the first result follows by rearranging this expression.
The limiting expression when $n,p\to \infty$ and $n/p\to a\in(0,\infty)$ follows from the fact that $\lambda_{\max}(X^\intercal X)/p \to (1+\sqrt{a})^2$~\citep[Theorem~2]{Bai1993}. For $Q_0=(X^\intercal X/g + c I_{p})$ we follow the same steps.
\end{proof}
Figure~\ref{fig:cavigs} illustrates these results numerically.

\begin{figure}[!tbp]
\caption{Log-Wasserstein-$2$ convergence of CAVI in Bayesian Probit Regression with a $g$-prior (averaged across 100 simulations). \textbf{Left Panel:} Convergence behaviour for a fixed $g$ across varying values of $(n, p)$. The rate of convergence is independent of both sample size and dimensionality. \textbf{Right Panel:} Convergence behaviour across varying values of $g$ at a fixed $(n, p)$. The convergence rate exhibits an inverse relationship with the magnitude of $g$.}
\label{fig:cavigs}
\centering
\begin{minipage}[b]{0.45\textwidth}
   \includegraphics[width=\textwidth]{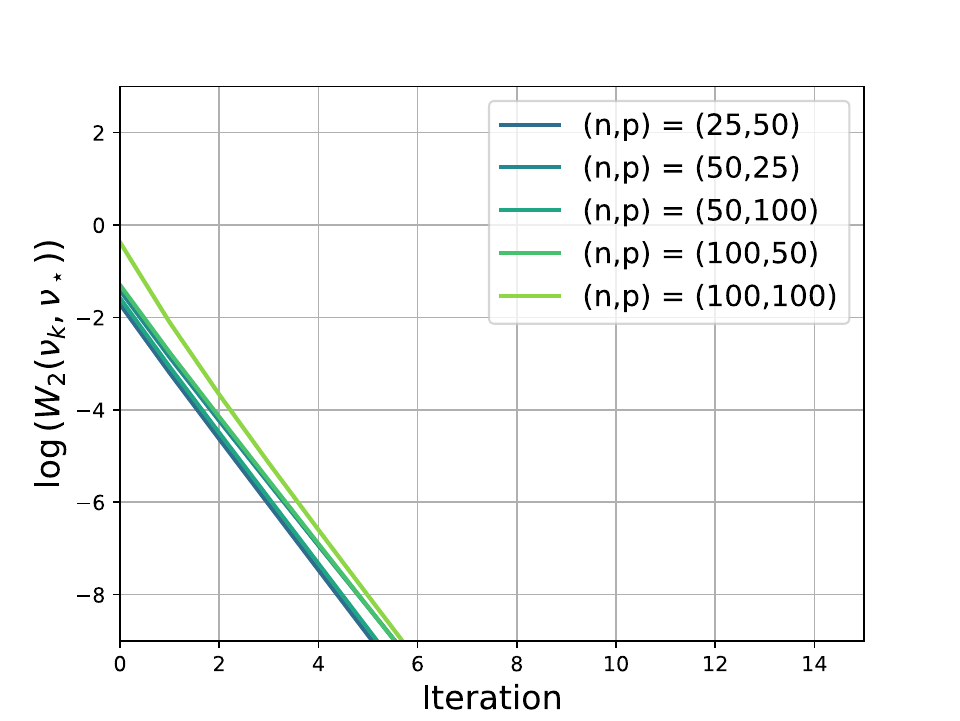}
\end{minipage}
\hfill
\begin{minipage}[b]{0.45\textwidth}
   \includegraphics[width=\textwidth]{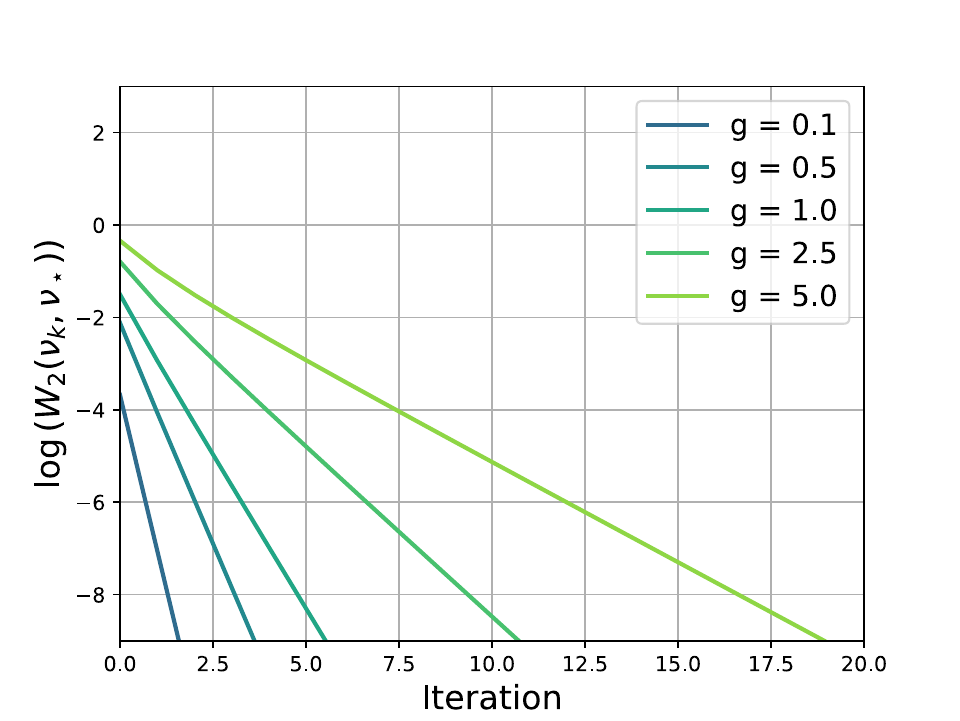}
\end{minipage}
\end{figure}

\paragraph{Logistic Regression.} \label{subsec:logistic}
Suppose now that $\Phi(x)$ is the logistic function. In this case, state-of-the-art methods consider \emph{Pólya--Gamma} augmentations $Z_i\sim \pg(1,x_i^\intercal \beta)$~\citep{Polson2013}.

The resulting augmented posterior $\pi(\dif z,\dif \beta)$ admits a two-block CAVI scheme given by
\begin{equation} \label{eq:cavilogit}
\begin{aligned}
    \mu_\star^1[\mu^2](\dif z ) &=  \bigotimes_{i=1}^n \pg(\dif z_i; 1, c_{i,\mu^2}), \quad \mu_\star^2[\mu^1](\dif \beta) &= \cal{N}(\dif \beta; m_{\mu^1},Q_{\mu^1}^{-1})
\end{aligned}
\end{equation}
with $c_{i,\mu^2}:= \left(\bb{E}_{\mu^2}[(x_i^\intercal\beta)^2]\right)^{1/2} \quad m_{\mu^1} = Q_{\mu^1}^{-1}m, \quad  Q_{\mu^1}=X^\intercal \diag(\bb{E}_{\mu^1}[Z_i])X + Q_0, \quad m=(X^\intercal(y-1/2) + Q_0 m_0)$. This coincides with the celebrated algorithm given in \citet{Jaakkola2000}, see \citet[Lemma 1]{Durante2019} for details on this equivalence.
The target distribution $\pi$ in this case is log-concave, but not strongly so. Nevertheless, our exponential convergence results apply.
Recall that $\bb{E}[Z]=\phi(c):=\tanh(c/2)(1/2c)$ when $Z\sim \pg(1,c)$~\citep{Polson2013}. Define the second moment matrix of $\mu^2_\star[\mu^1]$: $S(c)=Q(c)^{-1}+m(c)m(c)^\intercal$, where $Q(c)=X^\intercal \diag(\phi(c_i)) X + Q_0$, $m(c)=Q(c)^{-1}m$.
We have the following convergence result.
\begin{proposition}\label{cor:cavibayeslogit}
Let $\varepsilon > 0$, let $J_{i,\varepsilon}:=\biggl[(c_{i,\star}-\sqrt{x_i^\intercal Q_\star^{-1} x_i}\varepsilon)_+,\,(c_{i,\star}+\sqrt{x_i^\intercal Q_\star^{-1} x_i}\varepsilon)\biggr]$, and define
\begin{align*}
    A_{\varepsilon}:=Q_\star^{-1/2}X^\intercal \diag\Bigg( \frac{\sup_{c\in \times_{i=1}^n J_{i,\varepsilon}}|\phi'(c)|^2x_i^\intercal S(c)x_i}{\phi(c_{i,\star})}\Bigg) X Q_\star^{-1/2}.
\end{align*}
Suppose that the starting distribution satisfies $\mu_0^2\in \ball{\mu_\star^2}{\varepsilon}$. Then, the iterate $(\mu_k^2)_{k\geq 0}$ from the CAVI algorithm~\eqref{eq:cavilogit} for Bayesian Logit Regression converges exponentially fast in Wasserstein-$2$ distance with rate $r_\varepsilon$ if
\begin{align} \label{eq:cavilogitrate}
    r_\varepsilon^2:=\lambda_{\max}(A_\varepsilon) <1.
\end{align}
\end{proposition}
\begin{proof}
We apply Theorem~\ref{thm:secondcoordinate} with $p=2$, working in the norm induced by $Q_\star$ on $\r^p$. Log-concavity of $\mu_\star^2$ and Lemma~\ref{lemma:useful} give the transportation-information inequality with constant $\lambda_{2}=1$ in this norm.
It remains to find a $L_2>0$ such that $\cal{I}(\mu_\star^2[\mu_\star^1[\mu^2]]\parallel \mu_\star^2)\leq L_2 \bb{W}_2^2(\mu^2,\mu_\star^2)$ when $\mu^2\in \ball{\mu_\star^2}{\varepsilon}$.

For this proof, define $T(\mu^2)=\mu_\star^2[\mu_\star^1[\mu^2]]$ (so that $T(\mu^2)=\cal{N}(m(c_{\mu^2}),Q(c_{\mu^2})^{-1})$), $D(c_{\mu^2}):=\diag(\phi(c_{i,\mu^2}))$, $\Delta Q:=X^\intercal (D(c_{\mu^2})-D(c_\star))X$, and finally
\begin{align*}
	a_i &:= \sup_{c\in\times_j J_{j,\varepsilon}}\frac{|\phi'(c)|^2\,x_i^\intercal S(c)\,x_i}{\phi(c_{i,\star})},
	\qquad
	A_\varepsilon := Q_\star^{-1/2}X^\intercal\diag(a_i)\,X Q_\star^{-1/2}.
\end{align*}
Since $Q(c)m(c)=m=Q_\star m_\star$, we have
\begin{align*}
	\nabla_\beta \log \frac{T(\mu^2)}{\mu^2_\star}(\beta) = -(Q(c_{\mu^2})-Q_\star)\beta
\end{align*}
and it readily follows that the relative Fisher information in the $Q_\star$-metric is therefore
\begin{align*}
	\cal{I}(T(\mu^2) \parallel \mu_\star^2)
	&= \bb{E}_{T(\mu^2)}\bigl[\beta^\intercal\Delta Q Q_\star^{-1}\Delta Q \beta\bigr]
	= \tr\bigl(Q_{\star}^{-1} \Delta Q S(c_{\mu^2}) \Delta Q \bigr)
\end{align*}
Applying cyclicity of the trace and then the matrix bound $XQ_\star^{-1}X^\intercal\preceq D_\star^{-1}$, we obtain that
\begin{align*}
    \cal{I}(\mu_\star^2[\mu_\star^1[\mu^2]] \parallel \mu_\star^2)
    = \tr\bigl(D_\Delta X S(c_{\mu^2}) X^\intercal D_\Delta XQ_\star^{-1}X^\intercal\bigr) \nonumber
    \leq \sum_{i=1}^n \frac{(\phi(c_{i,\mu^2})-\phi(c_{i,\star}))^2}{\phi(c_{i,\star})} x_i^\intercal S(c_{\mu^2}) x_i.
\end{align*}
Now recall $c_{i,\mu^2}:=(\bb{E}_{\mu^2}[(x_i^\intercal\beta)^2])^{1/2}$.
For any coupling of $\mu^2$ and $\mu^2_\star$, the $\cal{L}^2$ reverse triangle inequality gives $|c_{i,\mu^2}-c_{i,\star}|\leq\bb{E}[|x_i^\intercal(\beta-\beta')|^2]^{1/2}$. Optimising over couplings and applying the Cauchy--Schwarz inequality with respect to the $Q_\star$ metric, we obtain that
\begin{align*}
    |c_{i,\mu^2}-c_{i,\star}|
    \leq (x_i^\intercal Q_\star^{-1}x_i)^{1/2} \bb{W}_2(\mu^2,\mu^2_\star)
    < (x_i^\intercal Q_\star^{-1}x_i)^{1/2} \varepsilon,
\end{align*}
so that $c_{i,\mu^2}$ lies in the interval $J_{i,\varepsilon}$ for all $\mu^2\in \ball{\mu_\star^2}{\varepsilon}$.
By the mean value theorem, $\phi(c_{i,\mu^2})-\phi(c_{i,\star})=\phi'(\xi_i)(c_{i,\mu^2}-c_{i,\star})$ for some $\xi_i\in J_{i,\varepsilon}$.
Substituting the mean value bound into the preceding bound for $\cal{I}$ and taking the optimal coupling of $\mu^2$ and $\mu^2_\star$, we obtain that
\begin{align*}
     \cal{I}(\mu_\star^2[\mu_\star^1[\mu^2]] \parallel \mu_\star^2)
    &\leq \sum_{i=1}^n a_i\bb{E}\bigl[|x_i^\intercal(\beta-\beta')|^2\bigr]
    = \bb{E}\bigl[(\beta-\beta')^\intercal X^\intercal\diag(a_i) X(\beta-\beta')\bigr] \\
    &\leq \lambda_{\max}(A_\varepsilon) \bb{E} \bigl[(\beta-\beta')^\intercal Q_\star(\beta-\beta')\bigr]
    = \lambda_{\max}(A_\varepsilon) \bb{W}_2^2(\mu^2,\mu_\star^2)
\end{align*}
as required.
\end{proof}
The statement of Proposition~\ref{cor:cavibayeslogit} requires an $\varepsilon$-close start, and this is related to the lack of global strong log-concavity of the posterior in logit regression~\citep{Chak2026}. We now study in detail the asymptotic case where $\varepsilon\to 0$. This makes clear that the convergence rate in Proposition~\ref{cor:cavibayeslogit} does not degrade when the number of data points $n$ and/or the dimensionality $p$ is large. In fact, since
\begin{align*}
c_{i,\star}^2 = \bb{E}_{\mu_\star^2}[(x_i^\intercal \beta)^2] = (x_i^\intercal m_\star)^2 + x_i^\intercal Q_\star^{-1}x_i = x_i^\intercal S(c_{\star})x_i,
\end{align*}
we see that the implied asymptotic convergence rate as $\varepsilon \to 0$ is characterised as
\begin{align} \label{eq:cavilogitasymptrate}
r_\star^2
&= \lambda_{\max}\Big(Q_\star^{-1/2}X^\intercal \diag\Big(\frac{\phi'(c_{i,\star})^2 c_{i,\star}^2}{\phi(c_{i,\star})}\Big) X Q_\star^{-1/2}\Big).
\end{align}
To our knowledge, this is the first available convergence rate result for the Jaakkola--Jordan algorithm. The closest result in spirit we are aware of is \citet{Lee2024}'s mixing time result for the Gibbs sampler.

The following is an analogue of Corollary~\ref{cor:cavibayesprobitrates} for logistic regression, and examines two relevant choices for the prior precision in Bayesian statistics: the g-prior, and a direct scaling of the prior precision in $p$.
\begin{corollary} The following statements hold true.
\begin{itemize}
    \item Let $Q_0=(p/c)I_p$. Then,
    \begin{align*}
        r_\star^2 \leq \frac{(c/p)\lambda_{\max}(X^\intercal X)}{4+(c/p)\lambda_{\max}(X^\intercal X)} < 1.
    \end{align*}
    \begin{align*}
        \lim r_\star^2 \leq  \frac{c(1+\sqrt{a})^2}{4+c(1+\sqrt{a})^2} < 1
    \end{align*}
    \item  Let $Q_0=(X^\intercal X/g + c I_{p})$. Then,
    \begin{align*}
        r_\star^2 \leq \frac{\lambda_{\max}(X^\intercal X)}{(1+4/g)\lambda_{\max}(X^\intercal X)+4c} < 1.
    \end{align*}
    In particular, suppose that $X_{i,j}$ are i.i.d.~with unit variance. If $n,p\to \infty$ and $n/p\to a\in(0,\infty)$, then we have
    \begin{align*}
        \lim r_\star^2 \leq  \frac{1}{(1+4/g)} < 1
    \end{align*}
\end{itemize}
\end{corollary}
\begin{proof}
Let $Q_0=(p/c)I_d$. From~\eqref{eq:cavilogitasymptrate}, use the inequality $|\phi'(t) t| < \phi(t)$ and the expression $Q_\star=X^\intercal \diag(\phi(c_{i,\star}))X+Q_0$ to obtain that
\begin{align*}
    r_\star^2 &\leq \lambda_{\max}\Big(Q_\star^{-1/2}X^\intercal \diag(\phi(c_{i,\star})) X Q_\star^{-1/2}\Big) \\
    &=\lambda_{\max}\Big( ((p/c)I_d+X^\intercal \diag(\phi(c_{i,\star})) X)^{-1} X^\intercal \diag(\phi(c_{i,\star})) X \Big)
\end{align*}
Similarly to the probit case~\eqref{eq:sad89ahdsbfvf}, we apply the spectral mapping Theorem with $h(A):=((p/c)I_d+A)^{-1}A$, obtaining that
\begin{align*}
    r_\star^2 \leq \max_{\lambda\in\sigma(X^\intercal \diag(\phi(c_{i,\star})) X)}h(\lambda)= \frac{\lambda_{\max}(X^\intercal \diag(\phi(c_{i,\star}))X)}{\lambda_{\max}(X^\intercal \diag(\phi(c_{i,\star}))X)+(p/c)}.
\end{align*}
Since $\phi(x)\leq 1/4$, we have $X^\intercal \diag(\phi(c_{i,\star}))X\preceq X^\intercal X / 4$ and the first result readily follows.
Similarly to Corollary~\ref{cor:cavibayesprobitrates}, the limiting expression when $n,p\to \infty$ and $n/p\to \gamma\in(0,\infty)$ follows from the fact that $\lambda_{\max}(X^\intercal X)/p \to (1+\sqrt{\gamma})^2$~\citep[Theorem~2]{Bai1993}. For $Q_0=(X^\intercal X/g + c I_{p})$ we apply the same reasoning.
\end{proof}

Notably, the bounds for CAVI in logit problems suggest that the algorithm is faster to fit binary regression models with logistic link than with probit link; Figure~\ref{fig:cavilogit} corroborates these findings numerically.
\begin{figure}[!tbp]
\caption{Log-Wasserstein-$2$ convergence of CAVI in Bayesian Logistic Regression with a $g$-prior (averaged across 100 simulations). \textbf{Left Panel:} Convergence behaviour for a fixed $g$ across varying values of $(n, p)$. The rate of convergence is independent of both sample size and dimensionality. \textbf{Right Panel:} Convergence behaviour across varying values of $g$ at a fixed $(n, p)$. The convergence rate exhibits an inverse relationship with the magnitude of $g$. The  convergence speed of CAVI in logit regression is greater than in probit problems, as predicted by our results. }
\label{fig:cavilogit}
\centering
\begin{minipage}[b]{0.45\textwidth}
   \includegraphics[width=\textwidth]{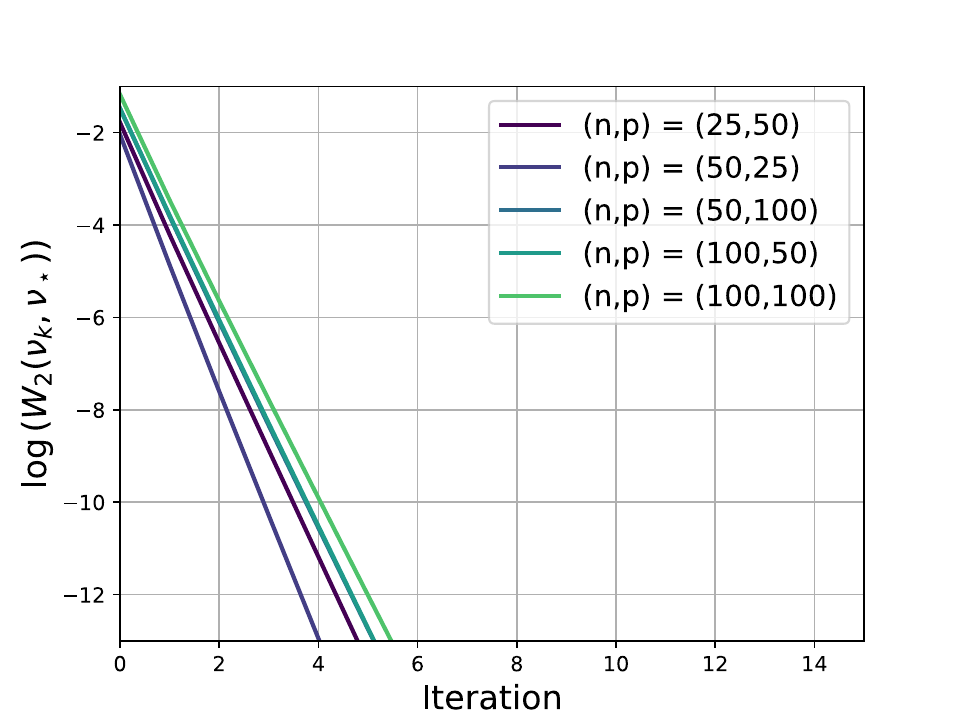}
\end{minipage}
\hfill
\begin{minipage}[b]{0.45\textwidth}
   \includegraphics[width=\textwidth]{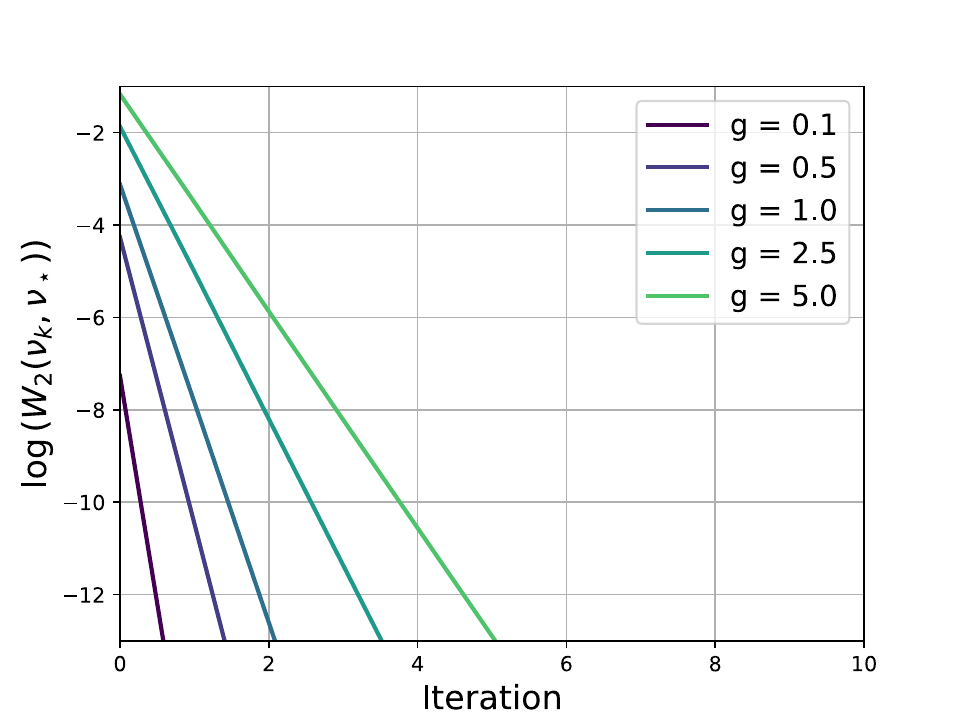}
\end{minipage}
\end{figure}

\subsection{Pairwise Markov Random Fields} \label{sec:pmrf}
Let $(\cal{V},\cal{E})$ be a finite graph with $n$ vertices, and let $\partial i = \{j\in\cal{V}:(i,j) \in \cal{E}\}$ denote the neighbours of the vertex $i$. Fix a set $\mathcal{X}$, and consider spin systems with state space $\mathcal{X}^{n}$ and a probability distribution proportional to 
\begin{align} \label{eq:dnfb333}
	\pi(x_1,\dots,x_n)\propto \exp\left(-\sum_{i\in\mathcal{V}}V_{i}\left(x_{i}\right)-\sum_{\left(i,j\right)\in\mathcal{E}}W_{i,j}\left(x_{i},x_{j}\right)\right),
\end{align}
where $\{ V_{i}\} $ denote `site potentials' and $\{ W_{i,j}\} $ denote `interaction potentials'. Such pairwise Markov random fields models form a rich class of examples in statistical physics~\citep{Binder2012}. We consider a coordinate ascent variational inference approach to approximate~\eqref{eq:dnfb333} by a per-vertex factorised probability distribution of the form $\mu^1\otimes\dots\otimes \mu^n\in\cal{P}^{\otimes n}(\cal{X})$. The CAVI updates read
\begin{align} \label{eq:bvbvbv6s}
	\mu^i_\star[\nu^{-i}](x_i) \propto \exp\Big(-V_i(x_i)-\sum_{j\in \partial i}\bb{E}_{\nu^j}[W_{ij}(x_i,X_j)] \Big).
\end{align} 
We consider the case of discrete and continuous $\cal{X}$. In the discrete setting,  we consider the Ising and dense Curie--Weiss models. The same framework and similar straightforward computations yield results for other structured graphs, such as bipartite graphs (as relevant for restricted Boltzmann machines). In the continuous setting, we consider the case of Gaussian free fields, relaxations thereof, and the Ginzburg--Landau model.

\subsubsection{Discrete spin systems}

\paragraph{Ising model.} 

In the Ising model, we have $\mathcal{X}=\{ \pm1\} $, $V_{i}(x_{i})=h_{i}\cdot x_{i}$
for some external field $h_{i}$, $W_{i,j}\left(x_{i},x_{j}\right)=J_{i,j}\cdot x_{i}\cdot x_{j}$, and the edges connects only neighbouring vertices.  
By Lemma~\ref{lemma:usefuldiscrete}, we know that $\lambda_i= 4$ for all $i\in [n]$. Moreover, we have
\begin{align*}
	\Delta^{(2)} = J_{ij}[x_i'x_j'+x_i''x_j''-x_i'x_j''-x_i''x_j']=J_{ij}(x_i'-x_i'')(x_j'-x_j'')
\end{align*}
and since $x_i',x_i'',x_j',x_j''\in\{-1,+1\}$, the largest possible absolute value of $\Delta^{(2)}$ is $4|J_{ij}|$ and thus $\delta_{ij}^{(2)}(f)=4|J_{ij}|$ and $L_{ij}= |J_{ij}|$ (Lemma~\ref{lemma:discretecrosssmooth}). Consider then the rate matrix $R$ with elements $R_{ij}=|J_{ij}|$ if $i\neq j$ and $0$ otherwise. 
An immediate application of Theorem~\ref{thm:cavidiscrete} yields the following.
\begin{proposition}
	The sequential, random, and parallel coordinate ascent variational inference algorithm for the Ising model converges globally fast in total variation distance with rates given by Theorems~\ref{thm:Dparallel},~\ref{thm:Drandom}, or~\ref{thm:Dsequential} if $\rho(R)<1$.
\end{proposition}
In particular, this shows that strong interactions between nodes slow down the coordinate ascent variational inference algorithm, other than making the variational approximation itself poor. This is a manifestation of the accuracy-convergence duality (cf. Remark~\ref{remark:accconv}). A well-known remedy, at least for Gibbs sampling, consists in blocking the vertices together~\citep{Swendsen1987}, as to attenuate the strength of inter-node interactions. Note that similar results are obtained in e.g.\ \citet{jain2018mean}.

\paragraph{Dense Curie--Weiss model.} In the dense Curie--Weiss model, we take $\mathcal{X}=\{ \pm1\} $, $V_{i}(x_{i})=h_{i}\cdot x_{i}$, $W_{i,j}(x_{i},x_{j})=J\cdot x_{i}\cdot x_{j}$, with $J=(\beta/n)$ for some $\beta>0$, and we consider a complete graph. 
From the previous computation in the Ising model, we know $\delta^{(2)}_{ij}(f)=4|J|$ when $i$ and $j$ are connected. Since the graph is connected,  in the Curie--Weiss model we then have $\delta^{(2)}_{ij}(f)=4\beta/n$ for all $i\neq j$. 
\begin{proposition}
	The sequential, random, and parallel coordinate ascent variational inference algorithm for the dense Curie--Weiss model converges globally fast in total variation distance with rates given by Theorems~\ref{thm:Dparallel},~\ref{thm:Drandom}, or~\ref{thm:Dsequential} if $|\beta|\smash{\frac{n-1}{n}}<1$.
\end{proposition}
Similar considerations to the Ising case also apply here. We note that in this instance, the influence-based convergence criterion correctly detects the known phase transition at $\beta = 1$; see e.g.\ Chapter 2 of \citet{friedli2017statistical}. 

\subsubsection{Continuous spin systems}

Consider now continuous-state spins $\cal{X} = \r^d$. These CAVI updates here might not be available in closed form for any choice of site and interaction potentials we consider next. In these cases, we understand the resulting algorithm as an idealised scheme that might be approximated with various methods~\citep{Wang2013}. 

\paragraph{Gaussian free field.} Consider the simplest setting where $V_i(x_i)=(\tau/2) |x_i|^2$ and $W_{ij}(x_i,x_j)=(J_{ij}/2)|x_i-x_j|^2$, $J_{ij}>0$. From the updates~\eqref{eq:bvbvbv6s} we observe $\lambda_i=\tau + \sum_{j\in \partial i} J_{ij}$ (Lemma~\ref{lemma:useful}), and by Lemma~\ref{lemma:crosssmooth} we have $L_{ij}=J_{ij}$. Consider the resulting rate matrix $R_{ij}=L_{ij}/\lambda_i$ $i\neq j$ and $R_{ii}=0$ as per Section~\ref{sec:Mblocks}. Notice that there always holds $\rho(R)<1$, and so any CAVI algorithm is guaranteed to be globally and exponentially convergent.
\begin{proposition}
	The sequential, random, and parallel coordinate ascent variational inference algorithm for the Gaussian free field model converges globally fast in Wasserstein-$1$ distance with rates given by Theorems~\ref{thm:Dparallel},~\ref{thm:Drandom}, or~\ref{thm:Dsequential}.
\end{proposition}

\paragraph{General log-concave site potentials and interactions.} 
Consider the more general case where $V_i$ is $\tau_i$-strongly convex, $W_{ij}$ is $\alpha_{ij}$-convex for some $\alpha_{ij}\geq 0$, and the mixed derivatives of $W_{ij}$ are bounded $\parallel\nabla_{x_i,x_j}^2W_{ij}(x_i,x_j)\parallel_{op}\leq \beta_{ij}$. By Lemma~\ref{lemma:crosssmooth} we have $L_{ij} = \beta_{ij}$ and Lemma~\ref{lemma:useful} gives $\lambda_i=\tau_i+ \sum_{j\in \partial i} \alpha_{ij}$.  As opposed to the Gaussian free field case, notice that, since $\beta_{ij}\geq \alpha_{ij} \geq 0$, it is not necessarily the case that $\rho(R)<1$ for the resulting rate matrix. It is required to have additional convexity from the site potentials for global exponential convergence.
\begin{proposition}
	The sequential, random, and parallel coordinate ascent variational inference algorithm for the general log-concave site potentials and interactions model converges globally fast in Wasserstein-$1$ distance  with rates given by Theorems~\ref{thm:Dparallel},~\ref{thm:Drandom}, or~\ref{thm:Dsequential} if $\rho(R)<1$.
\end{proposition}

\paragraph{Ginzburg--Landau model.} It is also possible to consider non-convex site potentials, provided the inter-spins interactions are sufficiently mild. Consider an $n$-regular graph for simplicity, and the Ginzburg--Landau model, where the site potentials are of double-well type, namely $V_i(x_i) = V(x_i) = (g/4)|x|^4 -(r/2) |x|^2$, $g, r > 0$ and $W_{ij} (x_i, x_j) = (J/2) |x_i - x_j|^2$~\citep{Grunewald2009,Menz2013}. We always have $L_{ij} = J$ (Lemma~\ref{lemma:crosssmooth}). Moreover, as shown below, $\mu^i[\nu^{-i}]$ satisfies an (uniform in $\nu^{-i}$) Logarithmic Sobolev inequality, with a constant $C (J, g, r, n)$ that does not degenerate as $J\to 0$ (and that we make explicit below), and thus satisfies a transportation-information inequality with at least the same constant (Lemma~\ref{lemma:useful}). Consider the resulting rate matrix $R_{ij} = J \cdot C( J, g, r, n)^{-1}$, $i \sim j$ and $R_{ii} = 0$ as per Section~\ref{sec:Mblocks}. When the interaction strength $J$ is small enough, one always has that $\rho(R) < 1$.

\begin{proposition}
	The sequential, random, and parallel coordinate ascent variational inference algorithm for the Ginzburg--Landau model converges globally fast in Wasserstein-$1$ distance  with rates given by Theorems~\ref{thm:Dparallel},~\ref{thm:Drandom}, or~\ref{thm:Dsequential} if $\rho(J)=\rho(R)<1$.
\end{proposition}
\begin{proof}
    It suffices to show that $\mu^i[\nu^{-i}]$ satisfies an (uniform in $\nu^{-i}$) Logarithmic Sobolev inequality with good constant. In fact, for an arbitrary $\alpha>0$, we can write $V(x_i)=U_\alpha(x_i)+B_\alpha(x_i)$, where
    \begin{align*}
        U_\alpha(x_i)&=\frac{\alpha}{2}|x_i|^2+\frac{J}{2}\sum_{j\sim i}|x_i-x_j|^2 + \varphi_\alpha(x_i)\vee 0, \quad B_\alpha(x_i)=\varphi_\alpha(x_i) \wedge 0, \quad \textup{where} \\\varphi_\alpha(x_i)&=\frac{g}{4}|x_i|^4 - \frac{\alpha+r}{2}|x_i|^2.
    \end{align*}
    Note first that $\varphi_\alpha \vee 0$ is convex (a picture helps for this), and hence deduce that $x_i\mapsto U_\alpha(x_i)$ is $(\alpha + nJ)$-strongly convex, while $x_i\mapsto B_\alpha(x_i)$ is an upper- and lower-bounded function. Moreover, $\osc(B_\alpha(x_i)) \leq \smash{\frac{(r+\alpha)^2}{4g}}$, and the celebrated criterion of \citet{Holley1987} shows that $\mu^i[\nu^{-i}]$ satisfies the Logarithmic Sobolev inequality with constant at least
    \begin{align*}
        C(J,g,r, n) \geq \frac{\alpha + n J}{4}\exp\left(-\frac{(r+\alpha)^2}{4g} \right)
    \end{align*}
    One then computes that $\sum_{j \sim i} R_{ij} \leq \frac{4 n J}{\alpha + nJ } \exp \left( \frac{(r+\alpha)^2}{4g} \right)$, which is indeed smaller than $1$ for $J$ sufficiently small; we conclude.
\end{proof}

\paragraph{Latent Gaussian Spatial Model.} Consider now a spatial model with a latent Gaussian random field prior and conditionally independent emissions. That is, at spatial locations $\left\{ x_i : 1\leq i \leq n \right\}$, model a latent field $f_i = f (x_i)$ as $f \sim \mathcal{N} \left( 0, Q \right)$ for some precision matrix $Q$, and then observe conditionally-independent emissions at each location according to some $p\left(y_{i}\mid f_{i}\right)$. For concreteness, take $y_i \mid f_i \sim \mathsf{Poisson} \left( \exp f_i \right)$, and consider posterior inference on $f$ given $y$. 

Observe that for this model, the negative log-likelihood of each observation has curvature which can be both arbitrarily large and arbitrarily close to zero, impeding various other theoretical analyses. With this being said, an influence-based analysis remains feasible, as this sharp growth involves only individual terms $f_i$, and not interactions. Straightforward calculations allow for us to bound $R_{ij} \leq \frac{|Q_{ij}|}{Q_{ii}}$ (just as one would obtain for studying the prior distribution directly) and convergence guarantees follow. 

Precisely the same analysis goes through in the case of binary emissions with logistic link, or indeed any log-concave emission distribution. In the case of binary emissions with probit link, or continuous emissions with Gaussian law, the likelihood supplies some additional curvature, and the rates improve accordingly. 

Note that for `non-spatial' priors (e.g.\ conventional generalised linear models with prior $\beta \sim \mathcal{N} \left( 0, \sigma^2 I_p \right)$ and observations of the form $y_i \mid x_i, \beta \sim p\left(y_{i}\mid \eta_{i} \right)$, $\eta_{i} = x_i^\top \beta$), the same analysis becomes much more challenging due to the couplings between the coordinates of $\beta$ induced by the design matrix $\mathbf{X}$. This illustrates the natural principle that before seeking a mean-field approximation, one should take care to work in a parameterisation of the model under which weak dependence is a credible assumption.

\paragraph{Hidden Markov Model.} Consider now a Markov model in discrete space and time, i.e.\ some initial state $x_0$ is drawn from some discrete distribution $m_0$, and then for $0 < t \leq T$, $x_t$ is drawn from some Markov kernel $M \left(x_{t-1},\cdot \right)$. Assume that the kernel $M$ is rather insensitive to its source and target locations in the sense that for some finite $\delta_{\mathrm{f}}, \delta_{\mathrm{b}} > 0$, it holds that for each $(x_1, x_2)$ (resp. $(x_1^\prime , x_2^\prime )$) that
\begin{align*}
    \mathrm{osc} \left(\log\frac{M\left(x_{1},\cdot\right)}{M\left(x_{2},\cdot\right)}\right)\leq\delta_{\mathrm{f}},\qquad\mathrm{osc}\left(\log\frac{M\left(\cdot,x^{\prime}_{1}\right)}{M\left(\cdot,x^{\prime}_{2}\right)}\right)\leq\delta_{\mathrm{b}}.
\end{align*}
Under these assumptions, one computes that the optimality mapping $\mu_\star^t$ is discrete-Fisher-smooth with constants $\left( L_{t,t-1}, L_{t, t+1} \right) = \left( \delta_\mathrm{b}, \delta_\mathrm{f} \right)$ respectively, with all other smoothness constants vanishing. One then sees that as soon as $\delta_\mathrm{b} + \delta_\mathrm{f} < 1$, the CAVI iterations will contract to their fixed point exponentially fast.

A natural first response is that for interesting Markov models, the dependence along iterates of the process should be somewhat stronger, that the mean-field approximation is ill-founded, and that the CAVI iterations should converge slowly; all of these points are essentially correct. A productive solution is to then make a \emph{structured} mean-field approximation by \emph{temporal blocking}, i.e.\ partitioning the time horizon $\left\{ 0, 1, \ldots, T\right\}$ into non-overlapping blocks of length $1 \ll B \ll T$, and modelling the $x_t$ within each block as being dependent, but independent of the $x_t$ in other blocks. This can work well. In particular, if one takes $B$ to be on the order of the decorrelation time of the process, upon working with the Hamming metric, one can establish bounds on the influence coefficients which are able to imply non-trivial convergence bounds for CAVI; we omit the details.

Similar calculations also go through for linear-Gaussian Markov models with contractive dynamics, much as one might expect.  Analogously to the Spatial GLM example, it is straightforward to extend these guarantees to the Hidden Markov Model setting, provided that the emissions admit log-concave densities. Our omission of the details for these comments is not particularly egregious as, given a sufficiently space-like interpretation of the time variable in a dynamical model, the claims already follow from the Spatial GLM case.

\section{Discussion}

In this work, we studied the contraction in the sense of Wasserstein distance of various Coordinate Ascent Variational Inference algorithms. The assumptions considered are a local functional smoothness condition for the optimality maps, and a functional inequality on their fixed points.  
We considered different selected examples in Bayesian statistics and statistical physics to illustrate out results. There is no shortage of further examples to which the framework can be readily applied, including regression models, exchangeable models of both mean-field and de Finetti type, inference on latent community structure in the stochastic block model, discrete point processes of various types, and well beyond. We are eager to see what other classes of problem can be fruitfully approached with the tools developed herein.

We are optimistic that the analysis developed herein will open up various interesting applications and extensions. One of some interest is to identify other tractable influence-based criteria which are used to study the mixing of Markov chains used for sampling systems of weakly-dependent random variables, and to examine the extent to which these techniques can be similarly informative about CAVI and related algorithms. In particular, the spectral independence framework of \citet{blanca2022mixing} has proven itself to be rather powerful for sharpening mixing bounds on various classical discrete systems, and it would be of great interest to extend those breakthroughs to the study of variational inference.

\paragraph{Dobrushin criterion and Gibbs sampling.} It is interesting to compare our approach to the analysis of coordinate ascent variational inference with the analysis of Gibbs sampling via the Dobrushin's criterion, on which we believe it provides some new insights. We anticipate some here and plan to discuss this in more depth in future research.

The Gibbs sampler  (also `Glauber dynamics', or `Heat bath') induces a $\pi$-invariant Markov process by repeatedly replacing its elements with draws from the full conditional distributions $\pi_i ( \cdot \mid x_{-i} )$, and has been tremendously impactful in Bayesian computation and beyond~\citep{Geman1984,Smith1993}. In studying the convergence to equilibrium of the Gibbs sampler, a pivotal contribution of Dobrushin~\citep[later refined by][and others]{Wu2006,Dyer2008,Wang2014} was to introduce the notion of `influence' coefficients $R$, defined as
\begin{align*}
    R_{ij} := \sup \left\{ \frac{\mathbb{W}_p \left( \pi_i ( \cdot \mid x_{-i} ), \pi_i ( \cdot \mid y_{-i}) \right) }{\mathsf{d} ( x_j , y_j )} : x_k = y_k \text{ for } k \neq j \right\}.
\end{align*}
Following derivations much like those presented in this work, one sees that if the matrix $R$ has spectral radius $\rho(R) < 1$, then the Gibbs sampler contracts towards $\pi$ exponentially quickly in Wasserstein distance. Such approaches have been particularly popular in the study of discrete-state spin systems, where they are able to supply dimension-independent estimates (although they often fail to accurately predict phase transitions, being slightly too blunt a tool for this).
On the other hand, this approach seems to have seen relatively limited application in the context of continuous spin systems, though perhaps without good cause. In particular, making the assumptions that
\begin{enumerate}
    \item for each $i$ and for each $x_{-i}$, the \emph{conditional} distribution $\pi_i \left( \cdot \mid x_{-i} \right)$ satisfies a transportation-information inequality with constant $\lambda_i$, and
    \item for each $i \neq j$, the gradient field $\nabla_{x_i} \log \pi$ is $L_{ij}$-Lipschitz with respect to $x_j$,
\end{enumerate}
one can in fact establish bounds of the form $R_{ij} \leq \frac{L_{ij}}{\lambda_i}$ when $p=1$, and hence deduce quantitative convergence rates for the associated Gibbs sampler in Wasserstein-1 distance. As a result, in many of our examples, while verifying the assumptions which ensure rapid convergence of CAVI, we obtain also simple convergence guarantees for the associated Gibbs sampler. In some of these cases, these convergence rates appear to be new. We notice that:
\begin{itemize}
    \item The criteria for CAVI and Gibbs are related but distinct. Our approach to CAVI hinges on the assumption that the marginal satisfies a transportation-information inequality, whereas the corresponding argument for the Gibbs sampler requires the conditionals to satisfy a uniform transportation-information. 
    \item The cross-smoothness condition above is also sufficient to verify globally the Fisher smoothness condition we consider for CAVI (Lemma~\ref{lemma:crosssmooth}), and thus investigate its global convergence (but not necessary at all for local guarantees). The concept of `local convergence' is very natural for CAVI, and much less so for Gibbs sampling.
    \item The set of models for which CAVI and Gibbs work well are again related but distinct. Even for the bivariate Ising model $\pi(x^1, x^2) \propto \exp ( \beta x^1 x^2 )$ for $x^1, x^2 \in \{ \pm 1 \}$, one sees that Dobrushin's criterion always guarantees exponential contraction for the Gibbs sampler on $\pi$, whereas CAVI can admit multiple fixed points for $\beta > 1$. As such, there is a genuine separation between the methods and the associated rates; while one should expect some qualitative concordance between when the two methods work, a more definitive statement does not obviously hold. We anticipate that this phenomenon is known to some degree, but we have not seen it discussed in-depth in the literature.
\end{itemize}

\section*{Acknowledgements}
This work was supported by the ProbAI Hub
(EP/Y028783/1). The first author is grateful to Giacomo Zanella for helpful discussions.

\printbibliography
\appendix
\section{Coordinate descent} \label{app:cd}
Our proof for the coordinate ascent variational inference algorithm admits a standard coordinate descent algorithm analogue that seems to be new, seems to work under different assumptions than those considered in the literature on analysis of coordinate descent algorithms, while yielding sharp rates of convergence.

The coordinate descent method is a fundamental optimization algorithm that minimises an objective function by iterating minimization procedures on the different blocks. Its convergence properties are classical~\citep{Ortega1970,Bertsekas1989}. More recently, various works have studied non-asymptotic quantitative bounds for convergence.
The algorithm with two blocks is known as alternating minimisation. \citet{Beck2013} establish error bounds for alternating minimisation under a global strong convexity assumption on $f$ and block-wise smoothness assumption.\citet{Karimi2016} shows convergence of coordinate descent with the Gauss-Southwell rule under the more general Polyak--{\L}ojasiewicz inequality. More recently, \citet{Both2022} improved on \citet{Beck2013}'s linear convergence results when considering a quasi-strongly convex objective $f$ enjoying block-wise smoothness.

Here, we prove linear convergence of sequential, parallel, and random scan coordinate descent, under cross-smoothness and a gradient quadratic growth condition for the optimal value functions. We also allow for local convergence guarantees.

Let $\{\cal{X}^i\}$, $i=1,\dots,M$ be smooth Riemannian manifolds, and let $f:\times_{i = 1}^M \cal{X}^i\mapsto \r$ be a given objective function we wish to minimise. Consider the optimality operators $x_\star^i:\times_{j \neq i}\cal{X}^j\mapsto \cal{X}^i$ given by
\begin{align*}
x_\star^i[x^{-i}] = \argmin_{x^i\in\cal{X}^i} f(x^{<i},x^i,x^{>i}).
\end{align*}
Let $(x_k^i)$, $i=1,\dots,M$ be the current iterate. We consider the sequential coordinate descent algorithm
\begin{align} \label{eq:CDsequential}
	x_{k+1}^i = \argmin_{x^i\in\cal{X}^i} f(x^{<i}_{k+1},x^i,x^{>i}_k) = x_\star^i[x^{<i}_{k+1},x^{>i}_k].
\end{align}
the parallel
\begin{align} \label{eq:CDparallel}
	x_{k+1}^i = \argmin_{x^i\in\cal{X}^i} f(x^{<i}_k,x^i,x^{>i}_k) = x_\star^i[x^{-i}_k].
\end{align}
and finally the random scan
\begin{align} \label{eq:CDrandom}
	x_{k+1}^I = \argmin_{x^I\in\cal{X}^I} f(x^{<I}_k,x^I,x^{>I}_k) = x_\star^I[x^{-I}_k], \quad x_{k+1}^{-I} = x_k^{-I}
\end{align}
where $I\sim\textup{Unif}(\{1,\dots,M\})$.
We wish to understand the convergence of these optimisation algorithms towards the fixed point $x_\star = (x_\star^1, \dots, x_\star^M)$, which satisfies
\begin{align*}
    x_\star^i=x_\star^i[x_\star^{-i}]
\end{align*}
which generally depends on the initialisation $x_0=(x_0^1,\dots,x_0^M)$.
The following result is the analogue to Theorem~\ref{thm:Dparallel}, and the proof is similar. Let $\mathsf{d}$ denote the geodesic distance, $\nabla$ and $|\cdot|$ the gradient and norms. Define the vector of geodesic distances to the minimiser of $f$ as
\begin{align*}
	D_{k} := \bigl(\mathsf{d}(x^1_k, x^1_\star), \dots, \mathsf{d}(x^M_k, x^M_\star)\bigr)^\top \in \r^M,
\end{align*}
and the rate matrix $R \in \r^{M \times M}$ by
\begin{equation*}
	R_{ij} := \frac{L_{ij}}{\lambda_i} \quad (i \neq j), \qquad R_{ii} := 0.
\end{equation*}
As in the coordinate ascent variational inference case, recall that $\rho(R)<1$ implies the existence of an eigenvector $\delta\in\r^M$ with $\rho(R)\delta\prec \delta$. Let $R_U$ be the strictly upper triangular part of $R$ and $R_L$ the lower triangular one.
\begin{theorem} \label{thm:cd}
    Let $\{\cal{X}^i\}_{i=1}^M$ be smooth Riemannian manifolds and assume the following.
    \begin{itemize}
    \item The function $f$ is cross-smooth in $\times_{i=1}^M \ball{\mu_\star^i}{\varepsilon_{i}}$: for all $i=1,\dots,M$, there exist constants $L_{ij}$ such that, for all $x_i\in \ball{x_\star^i}{\varepsilon_i}$,
        \begin{align} \label{eq:cs}
        |\nabla_i f(x^{<i},x^{i-1},x^{>i})-\nabla_i f(x_\star^{<i},x^i,x_\star^{>i})| \leq  \sum_{j\neq i} L_{ij} \mathsf{d}(x^j,x_\star^j)
        \end{align}
    \item For all $i=1,\dots,M$, the optimal value functions $x^i \mapsto f(x_\star^{<i},x^i,x_\star^{>i})$ satisfy a quadratic growth inequality with respect to the gradient: there exist constants $\lambda_i$ such that 
    \begin{align} \label{eq:qg}
    \lambda_{i}^2 \mathsf{d} (x^i,x_\star^i)^2 \leq |\nabla_i f(x_\star^{<i},x^i,x_\star^i,x_\star^{>i})|^2
    \end{align}
    \end{itemize}
    Suppose furthermore that $\rho(R)<1$ and consider any $c>0$ such that $c\delta_i\leq \varepsilon_i$. Then, for all $x_k^i\in \ball{x_\star^i}{c\delta_i}$ $i=1,\dots,M$,
    \begin{itemize}
        \item The sequential coordinate descent algorithm~\eqref{eq:CDsequential}  satisfies
        \begin{align*}
        D_{k+1} \preceq (I_M-R_L)^{-1}R_U D_k
        \end{align*}
        In particular, sequential coordinate descent algorithm~\eqref{eq:CDsequential} converges exponentially fast to $x_\star$ in geodesic distance with rate $\rho((I_M-R_L)^{-1}R_U)$.
        \item The parallel coordinate descent algorithm~\eqref{eq:CDparallel}  satisfies
        \begin{align*}
        D_{k+1} \preceq R D_k
        \end{align*}
        In particular, parallel coordinate descent algorithm~\eqref{eq:CDparallel} converges exponentially fast to $x_\star$ in geodesic distance with rate $\rho(R)$.
        \item The random coordinate descent algorithm~\eqref{eq:CDrandom}  satisfies
        \begin{align*}
        \mathbb{E} \left[ D_{k+1} \right] \preceq \Big(\frac{1}{M}R + \frac{M-1}{M}I_M\Big) D_k
        \end{align*}
        In particular, random coordinate descent algorithm~\eqref{eq:CDrandom} converges exponentially fast to $x_\star$ in expected geodesic distance with rate $\rho(\smash{\frac{1}{M}R + \frac{M-1}{M}I_M})$.
    \end{itemize}
\end{theorem}
\begin{proof}
    We prove the result for the parallel algorithm, the sequential and random-scan bounds follow by the same induction and expectation arguments as in Theorems~\ref{thm:Dsequential} and~\ref{thm:Drandom}.
    Suppose that $x_k^i\in \ball{x_\star^i}{c\delta_i}$ $i=1,\dots,M$. Then, for all $i=1,\dots,M$, by the gradient quadratic growth of $x^i\mapsto f(x_\star^{-i},x^i)$,
    \begin{align*}
     \lambda_i\mathsf{d}(x_{k+1}^i,x_\star^i) &\leq |\nabla_i f(x_\star^{<i},x_{k+1}^i,x_\star^{>i})| \\
     &= |\nabla_i f(x_\star^{<i},x_{k+1}^i,x_\star^{>i})-\nabla_i f(x_k^{<i},x^i_{k+1},x_k^{>i})|
    \end{align*}
    The last equality follows from the optimality of $x_{k+1}^i=\argmin_{x^i}f(x_k^1,\dots,x_k^{i-1},x^i,x_k^{i+1},\dots,x_k^M)$. Now we use the cross-smoothness assumption with a telescopic sum to obtain
    \begin{align*}
    	\lambda_i\mathsf{d}(x_{k+1}^i,x_\star^i) &\leq \sum_{j \neq i}L_{ij}\mathsf{d}(x_{k}^j,x_\star^j)
    \end{align*}
    Dividing by $\lambda_i$ shows $D_{k+1}\preceq R D_k$ by the definition of $R$ and we conclude. Moreover, $\mathsf{d}(x_{k+1}^i,x_\star^i)\leq \sum_{j\neq i}R_{ij}c\delta_j = c (R\delta)_i<c\delta_i$, implying $x_{k+1}^i\in \ball{x_\star^i}{c\delta_i}$.
\end{proof}
\begin{remark}
    A function $f$ is said to satisfy the  Polyak--{\L}ojasiewicz (\PL) inequality with constant $\lambda>0$ if, on its domain,
    \begin{equation} \label{eq:pl}
    2 \lambda (f-\inf f) \leq \norm{\nabla f}^2.
    \end{equation}
    The \PL inequality is a cornerstone of modern optimisation theory~\citep{Karimi2016,Garrigos2023}, it is weaker than strong convexity of $f$, and it implies the inequality $\lambda \mathsf{d}(u,\argmin f)^2 \leq 2(f(u)-\inf f)$.
    The gradient quadratic growth conditions ~\eqref{eq:qg} are weaker than \PL inequalities for the value function $x^i \mapsto f(x_\star^1, \dots, x_\star^{i-1}, x^i, x_\star^i, \dots, x_\star^M)$. In fact, these imply~\eqref{eq:qg}:
    \begin{align*}
    \lambda_{i}^2 \mathsf{d} (x^i,x_\star^i)^2 \leq 2\lambda_{i}(f(x_\star^{<i},x^i,x_\star^{>i})-\inf f) \leq |\nabla_i f(x_\star^{<i},x^i,x_\star^{>i})|^2.
    \end{align*}
    Moreover, \PL inequalities for the value function $x^i \mapsto f(x_\star^{<i},x^i,x_\star^{>i})$ are weaker than $f$'s strong convexity (see Remark~\ref{remark:logconc}). The gradient quadratic growth condition is also known as an error bound condition~\citep{Karimi2016}.
\end{remark}
\begin{remark}
    The cross-smoothness conditions is guaranteed globally on $\times_{i = 1}^M\cal{X}^i$ provided that the off-diagonal elements of $f$'s Hessian are bounded. In this case, the proof also simplifies, in that we do not have to worry about invariant neighbourhoods and to ensure iterates stay within.
\end{remark}
Theorem~\ref{thm:cd} allows to examine linear convergence for objective functions that violate classical strong convexity, \PL, and block-wise smoothness assumptions~\citep{Beck2013,Both2022}, while still being cross-smooth and satisfying the quadratic growth to the gradient at the fixed points~\eqref{eq:qg}. Following essentially the arguments of Remark~\ref{remark:gaussian}, it is possible to show that for quadratic objectives the rates are sharp.

\section{Remarks on transportation-information inequalities} \label{app:ti}

Here, we offer some additional background on transportation-information inequalities (hereafter TI), situating them within the broader landscape of functional inequalities, and giving some pointers to verify them. Throughout, write $\gamma$ for a fixed probability measure of interest, $\mathcal{L}$ for the infinitesimal generator of the associated Langevin diffusion, and $\left\{ P_t : t\geq 0\right\}$ for the associated transition kernels.
\begin{enumerate}
    \item Within the Bogachev--Kolesnikov hierarchy of geometric functional inequalities `of Gaussian type'~\citep[Section 3.5]{bogachev2012monge}, TI inequalities lie in between Sobolev-type energy-entropy inequalities and transport(-entropy) inequalities, being quantitatively weaker than the Logarithmic Sobolev inequality and quantitatively stronger than Talagrand-type inequalities~\citep[Proposition 2.9.a]{Guillin2009}.
    \item Analogously to transportation-entropy inequalities, in the case $p = 2$, a TI inequality still implies an energy-entropy inequality `of exponential type', i.e.\ the Poincar\'{e} inequality~\citep[Proposition 2.9.b]{Guillin2009}. Moreover, for log-concave $\gamma$ (or $\gamma$ which deviates from log-concavity sufficiently mildly), the TI inequality even yields a full Logarithmic Sobolev inequality~\citep[Proposition 2.9.c]{Guillin2009}. Examples are known which demonstrate that some condition of this form is necessary, i.e.\ that one can satisfy a TI inequality but not any Logarithmic Sobolev inequality; see Remark 2.6 of \citet{guillin2009ii}.
    \item In terms of verifying TI inequalities `from scratch', a rather crude but general tool is the method of Lyapunov conditions. If $\gamma$ is known a priori to satisfy a Poincar\'{e} inequality, then it is sufficient to exhibit a Lyapunov function $W \geq 1$, a reference point $x_0 \in \mathcal{X}$, and positive constants $b$, $c$ for which
    \begin{align*}
        \mathcal{L} W \leq b - c \, \mathsf{d} (x,x_0)^2 \, W.
    \end{align*}
    Under these conditions, a TI inequality holds for $\gamma$~\citep[Theorem 5.3]{Guillin2009}. If the Poincar\'{e} inequality is not known a priori for $\gamma$, then hope is not lost, as \citet{liu2017new} shows that it suffices to instead exhibit the stronger estimate
    \begin{align*}
        \mathcal{L} W \leq (b - c \, \mathsf{d} (x,x_0)^2 )\, W
    \end{align*}
    for some $W > 0$. The same work also establishes that such a condition is necessary for a TI inequality to hold.
    \item For quantitative purposes, it often works well to study the metric contraction properties of the Langevin diffusion. In \citet{wu2009gradient}, it is established that granted e.g.\footnote{Actually, \citet{wu2009gradient} makes the slightly more abstract assumption of a `Lipschitzian Spectral Gap' for $\mathcal{L}$: for some $c>0$ and for all sufficiently-regular $f$, it holds that $\| \mathcal{L} f \|_\mathrm{Lip} \geq c \| f \|_\mathrm{Lip}$. It is well-known that this is implied by the condition in the main text, and indeed, it is often proved by exactly these means.} hypocontractivity estimates of the form
    \begin{align*}
        \forall x, x^\prime \in \mathcal{X}, \qquad \mathbb{W}_1 ( P_t \left(x, \cdot \right), P_t (x^\prime, \cdot ) ) \leq \ell (t) \cdot \mathsf{d} ( x, x^\prime )
    \end{align*}
    with $\ell : \mathbb{R}_+ \to \mathbb{R}_+$ integrable, one can deduce a TI inequality for $\gamma$ with $p = 1$~\citep[Corollary 2.2]{wu2009gradient}, and the associated constant can be taken to depend only on $\| \ell \|_{\mathrm{L}^1} = \int_{\mathbb{R}_+} \ell (t)\,\mathrm{d}t$. As such, provided that the Langevin diffusion is `eventually uniformly exponentially contractive-on-average' in the Euclidean metric, a TI inequality follows.
    \item Thus far, the main applications of TI inequalities have apparently been in establishing path-space concentration inequalities for additive functionals along trajectories of the Langevin diffusion, following the developments in, e.g., Theorem 4.1 of \citet{Guillin2009}, and subsequently in \citet{gao2014bernstein}.~\cite{huggins2018practical} make some use of TI inequalities in validating approximations to Bayesian posterior distributions, but without focusing on how the approximations are computed, as we do in this work. 

\end{enumerate}

\section{Remarks on Functional Inequalities in Discrete Metric} \label{app:dfi}

We recall three functional inequalities of general interest. Fix a probability measure of interest $\gamma$ on $\mathbb{R}^d$ equipped with the Euclidean metric, consider an alternative probability measure $\varrho$, and introduce the writing the Fisher information, relative entropy, and Wasserstein distance as
\begin{align*}
    \iota = \mathcal{I} \left( \varrho, \gamma \right), \qquad \kappa = \mathrm{KL} \left( \varrho, \gamma \right), \qquad \tau = \mathbb{W}_2 \left( \varrho, \gamma \right).
\end{align*}
One can then consider for $\gamma\in\cal{P}(\r^d)$
\begin{enumerate}
    \item the Logarithmic Sobolev inequality, which entails that uniformly in $\varrho$, one can control the relative entropy by the Fisher information, i.e.\ for some finite $\lambda > 0$,
    \begin{align*}
        2\lambda \cdot \kappa \leq \iota 
    \end{align*}
    \item the transportation-information inequality, which entails that uniformly in $\varrho$, one can control the Wasserstein distance by the Fisher information, i.e.\ for some finite $\lambda > 0$,
    \begin{align*}
        \lambda^2 \cdot\tau^2 \leq \iota, 
    \end{align*}
    and,
    \item the transportation-entropy inequality, which entails that uniformly in $\varrho$, one can control the Wasserstein distance by the relative entropy, i.e.\ for some finite $\lambda > 0$,
    \begin{align*}
         \lambda^2 \cdot \tau^2 \leq \kappa.
    \end{align*}
\end{enumerate}
For probability measures on Euclidean spaces, equipped with the Euclidean metric, establishing any of these three inequalities tends to require one to establish some strong positive properties about the measure in question, such as curvature or concentration~\citep{Bakry2014}.

There is then some surprise involved in learning that for probability measures on \emph{arbitrary} spaces equipped with the \emph{discrete} metric, inequalities of this form hold generically, with universal constants. In particular, for any probability measure $\gamma$ and for any alternative probability measure $\varrho$ on a discrete space, writing the `discrete Fisher information', relative entropy, and total variation distance as
\begin{align*}
    \omega = \mathrm{osc} \left( \log \frac{\mathrm{d} \varrho}{\mathrm{d} \gamma} \right), \qquad \kappa = \mathrm{KL} \left( \varrho, \gamma \right), \qquad \tau = \mathsf{TV} \left( \varrho, \gamma \right)
\end{align*}
respectively, one can establish, following the arguments of Lemma~\ref{lemma:usefuldiscrete},
\begin{enumerate}
    \item A `discrete Logarithmic Sobolev inequality' $\kappa  \leq F(\omega)$, where 
    \begin{align*}
        F(\omega) = \frac{\omega}{1 - \exp(-\omega)} - 1 - \log \left( \frac{\omega}{1 - \exp(-\omega)} \right) \leq \frac{1}{8} \omega^2,
    \end{align*}
    \item A `discrete transportation-information inequality' $\tau \leq T(\omega)$, where
    \begin{align*}
        T(\omega) = \tanh \left( \frac{1}{4} \omega \right) \leq \frac{1}{4} \omega,
    \end{align*}
    and
    \item A `discrete transportation-entropy inequality' $\alpha \left( \tau \right) \leq \kappa$, where $\alpha$ is  a rather implicit function which can be described parametrically (see e.g.\ Section 7.5.2 of~\cite{polyanskiy2025information}) and can be tractably lower-bounded as
    \begin{align*}
         \alpha \left( \tau \right) \geq \log \left( \frac{1 + \tau}{1 - \tau} \right)  - \frac{2\cdot\tau}{1+\tau},
    \end{align*}
    or, more commonly, as $\alpha \left( \tau \right) \geq 2 \cdot \tau^2$; use of the latter lower bound is normally known as `Pinsker's inequality'.
\end{enumerate}
That each of these inequalities holds with universal constants is somehow testament to the bluntness of the discrete metric.

\end{document}